\documentclass{article}

\usepackage{a4wide}

\usepackage{times}
\usepackage{graphicx}
\usepackage{latexsym}

\usepackage{amssymb}
\usepackage{amsthm}
\usepackage{amsmath}

\usepackage{subfig}

\usepackage{todonotes}

\presetkeys{todonotes}{color=pink,inline}{}

\usepackage{tikz}
\usetikzlibrary{arrows}
\usetikzlibrary{shapes.geometric}

\newtheorem{example}{Example}
\newtheorem{definition}{Definition}
\newtheorem{proposition}{Proposition}

\newtheorem{assumption}{Assumption}

\newcommand{\co}{{\sf co}}
\newcommand{\pr}{{\sf pr}}
\newcommand{\gr}{{\sf gr}}
\newcommand{\st}{{\sf st}}

\newcommand{\reals}{\mathbb{R}}

\newcommand{\argin}{\ensuremath{\mathsf{in}}}		% argument is in
\newcommand{\argout}{\ensuremath{\mathsf{out}}}		% argument is out
\newcommand{\argundec}{\ensuremath{\mathsf{undec}}}		% argument is undec

\newcommand{\support}{{\sf Support}}
\newcommand{\nodes}{{\sf Nodes}}

\begin{document}

%\title{Measuring and Resolving Inconsistency\\ in Argument Graphs}

\title{Measuring Inconsistency in Argument Graphs}

\author{Anthony Hunter\\ Department of Computer Science,\\ University College London,\\ London, UK\\ anthony.hunter@ucl.ac.uk}

\maketitle

%%%%%%%%%%%%%%%%%%%%%%%%%%%%%%%%%%%%%%%%%%%%%%%%%%
%%%%%%%%%%%%%%%%%%%%%%%%%%%%%%%%%%%%%%%%%%%%%%%%%%
%%%%%%%%%%%%%%%%%%%%%%%%%%%%%%%%%%%%%%%%%%%%%%%%%%
%%%%%%%%%%%%%%%%%%%%%%%%%%%%%%%%%%%%%%%%%%%%%%%%%%
%%%%%%%%%%%%%%%%%%%%%%%%%%%%%%%%%%%%%%%%%%%%%%%%%%
%%%%%%%%%%%%%%%%%%%%%%%%%%%%%%%%%%%%%%%%%%%%%%%%%%

\begin{abstract}
There have been a number of developments in measuring inconsistency in logic-based representations of knowledge. In constrast, the development of inconsistency measures for computational models of argument has been limited. To address this shortcoming, this paper provides a general framework for measuring inconsistency in abstract argumentation, together with some proposals for specific measures, and a consideration of measuring inconsistency in logic-based instantiations of argument graphs, including a review of some existing proposals and a consideration of how existing logic-based measures of inconsistency can be applied.  
\end{abstract}

%%%%%%%%%%%%%%%%%%%%%%%%%%%%%%%%%%%%%%%%%%%%%%%%%%
%%%%%%%%%%%%%%%%%%%%%%%%%%%%%%%%%%%%%%%%%%%%%%%%%%
%%%%%%%%%%%%%%%%%%%%%%%%%%%%%%%%%%%%%%%%%%%%%%%%%%
\section{Introduction}

Argumentation is an important cognitive ability for handling conflicting and incomplete information such as beliefs, assumptions, opinions, and goals.  When we are faced with a situation where we find that our information is incomplete or inconsistent, we often resort to
the use of arguments for and against a given position in order to make sense of the situation. Furthermore, when we interact with other people we often exchange arguments to reach a final agreement and/or to defend and promote an individual position. 

In recent years, there has been substantial interest in the development of computational models of argument for capturing aspects of this cognitive ability (for reviews see \cite{BenchCapon07,BH08book,RS09}). This has led to the development of a number of directions including: (1) abstract argument models where arguments are atomic, and the emphasis is on the relationships between arguments; (2) logic-based (or structured) argument models where the emphasis is on the logical structure of the premises and claim of the arguments, and the logical definition of relations between arguments; and (3) dialogical argument models where the emphasis is on the protocols (i.e. allowed and obligatory moves that can be taken at each step of the dialogue) and strategies (i.e. mechanisms used by each participant to make the best choice of move at each step of the dialogue). 

At the core of computational models of argument is the ability to represent and reason with inconsistency. So it is perhaps surprising that relatively little consideration has been given to measuring inconsistency in these models, particularly given the number of developments in measuring inconsistency in logic-based knowledgebases (see for example 
\cite{Kni01,HK04mek,doder2010measures,HK10,GrantHunter2011,Ma2012jlc,Jabbour2014inconsistency,Besnard14,Thimm2016}). 
A couple of exceptions are the consideration of the  degree of undercut between an argument and counterargument \cite{BH05aaai,BH08book}, and measuring inconsistency through argumentation \cite{Raddaoui2015}. Note, the approach of weighted argumentation frameworks \cite{Dunne11} is not a measure of inconsistency as the approach assumes extra information (weights) to label each arc, and an inconsistency budget that allows arcs that sum to no more than the budget to be ignored.

There are a number of reasons why it is useful to investigate the measurement of inconsistency in argumentation:
(1) to better characterize the nature of inconsistency in argumentation;
(2) to analyse the inconsistency arising in specific argumentation situations;
and 
(3) to direct the resolution of inconsistency as arising in argumentation. 
We will consider contributions to these three areas during the course of this chapter. 

Given the central role of argument graphs (where each node is an argument and each arc denotes one argument attacking another) in modelling argumentation, we will consider the inconsistency of an argument graph. This is useful if we want to assess the overall conflict that is manifested by an argument graph, and we want to focus on actions that may allow us to decrease the graph inconsistency.

Consider for example some security analysts who are analyzing some conflicting reports concerning a foreign country that may be descending into civil war. These analysts may enter into a process as follows:
(1) they collect relevant information concerning the political and security situation in the country;
(2) they construct arguments from this information that draw tentative hypotheses about the situation in the country;
(3) they compose these arguments into an argument graph;
(4) they measure the inconsistency of the argument graph;
(5) they use the measure of inconsistency to identify information requirements (i.e. queries to ascertain whether particular argument should be accepted or rejected) therefore that would result in commitments being made for some of these arguments;
(6) they seek the answers to these queries; 
(6) they use these commitments to reduce the overall inconsistency of the graph;
and (7) they terminate this process when sufficient commitments have been made so as to reduce the inconsistency to a sufficiently low level.

%\todo[inline]{Expand on argument construction + graphs + commitment + info seek actions}

This kind of process may of relevance to security analysts to augment recent proposals for argument-based security analysis technology such as by Toniolo {et al} \cite{Toniolo2015}.  Furthermore, this kind of process may be replicated in roles such as business intelligence analysis, policy planning, political planning, and science research.

%\begin{figure}
%\begin{center}
%\begin{tikzpicture}[->,thick]
%\tikzstyle{every node}=[draw,circle,fill=yellow,minimum size=4pt,inner sep=0pt]
%\tikzstyle{every node}=[draw,rectangle,fill=yellow]
%			\node (A1)  at (1,1) {$A_1$};
%			\node (A2)  at (0,0) {$A_2$};
%			\node (A3) at (2,0) {$A_3$};
%			\path	(A1)[] edge[bend left] (A3);		
%			\path	(A3)[] edge[bend left] (A2);		
%			\path	(A2)[] edge[bend left] (A1);	

%\node (B1) at (4,0) {$B_1$};
%\node (B2) at (6,0) {$B_2$};
%\node (B3) at (8,0) {$B_3$};

%\path (B1) edge[->] (B2);
%\path (B2) edge[->] (B3);
%\draw[->] (B1) .. controls (5,1.3) and (3,1.3) .. (B1);

%\end{tikzpicture}
%\end{center}
%Aption{Example graphs}
%\end{figure}

We proceed as follows:
(Section \ref{section:abstract}) We review the basic definitions of abstract argumentation, considering both extension-based and label-based approaches; 
(Section \ref{section:graphinc}) We investigate a general framework for measuring inconsistency in abstract argumentation, together with some proposals for specific measures;
(Section \ref{section:logic}) We review deductive argumentation for instantiating abstract argument graphs, we review an existing proposal for measuring inconsistency in deductive argumentation called degree of undercut, and we investigate a new approach that harnesses existing logic-based measures;
(Section \ref{section:resolution}) We consider how we can use measures of inconsistency to direct the resolution of inconsistency in argumentation;
(Section \ref{section:discussion}) We conclude with a discussion of the proposals in the paper and of future work.

%%%%%%%%%%%%%%%%%%%%%%%%%%%%%%%%%%%%%%%%%%%%%%%%%%
%%%%%%%%%%%%%%%%%%%%%%%%%%%%%%%%%%%%%%%%%%%%%%%%%%
%%%%%%%%%%%%%%%%%%%%%%%%%%%%%%%%%%%%%%%%%%%%%%%%%%
\section{Review of abstract argumentation}
%%%%%%%%%%%%%
\label{section:abstract}

Our framework builds on more general developments in the area of computational models of argument. These models aim to reflect how human argumentation uses conflicting information to construct and analyze arguments. There is a number of frameworks for computational models of argumentation. They incorporate a formal representation of individual arguments and techniques for comparing conflicting arguments (for reviews see \cite{BenchCapon07,BH08book,RS09}). By basing our framework on these general models, we can harness theory and adapt implemented argumentation software as the basis of our solution.

%%%%%%%%%%%%%%%%%%%%%%%%%%%%%%%%%%%%%%%%%%%%%%%%%%
\subsection{Extension-based semantics}
%%%%%%%%%%%%%%%%%%%%%%%

We start with a brief review of abstract argumentation as proposed by Dung \cite{Dun95}. 
In this approach, each argument is treated as an atom, and so no internal structure of the argument needs to be identified. 

%An {\bf argument graph} is directed graph $G$.
%Let ${\sf Nodes}(G)$ be the set of nodes in $G$ 
%and let ${\sf Arcs}(G) \subseteq {\sf Arcs}(G) \times {\sf Arcs}(G)$ be the set of arcs in $G$.
%Each element $A \in {\sf Nodes}(G)$ is called an \textbf{argument} 
%and $(A,B) \in {\sf Arcs}(G)$ means that $A$ \textbf{attacks} $B$ 
%(accordingly, $A$ is said to be an {\bf attacker} of $B$)
%and so $A$ is a {\bf counterargument} for $B$.
%We let ${\cal G}$ denote the set of all argument graph (i.e. all directed graphs). 

%In this section, we review the proposal for abstract argumentation by Dung \cite{Dun95}. The simplest way to formalize a collection of arguments consists of just naming arguments (so, in a sense, treating them as atomic) and merely representing the fact that an argument is challenged by another (and so not indicating what the nature of the challenge is).In other words, a collection of arguments can be formalized as a directed binary graph.

\begin{definition}
An {\bf argument graph} is a pair $G = ({\cal A},{\cal R})$ where 
${\cal A}$ is a set and ${\cal R}$ is a binary relation over ${\cal A}$ 
(in symbols, ${\cal R} \subseteq {\cal A}\times {\cal A}$).
Let ${\sf Nodes}(G)$ be the set of nodes in $G$ 
and let ${\sf Arcs}(G) \subseteq {\sf Arcs}(G) \times {\sf Arcs}(G)$ be the set of arcs in $G$.
\end{definition}

So an argument graph is a directed graph. 
Each element $A \in {\cal A}$ is called an \textbf{argument} 
and $(A_i,A_j) \in {\cal R}$ means that $A_i$ \textbf{attacks} $A_j$ 
(accordingly, $A_i$ is said to be an {\bf attacker} of $A_j$).
So $A_i$ is a {\bf counterargument} for $A_j$ when $(A_i,A_j) \in {\cal R}$ holds.

\begin{example}
\label{ex:intro1}
Consider arguments $A_1$ = ``Patient has hypertension so prescribe diuretics",
$A_2$ = ``Patient has hypertension so prescribe beta-blockers",
and $A_3$ = ``Patient has emphysema which is a contraindication for beta-blockers".
Here, we assume that $A_1$ and $A_2$ attack each other because we should only give one treatment and so giving one precludes the other, and we assume that $A_3$ attacks $A_2$ because it provides a counterargument to $A_2$. Hence, we get the following abstract argument graph.
%\[
%A_1 \leftrightarrows A_2 \leftarrow A_3
%\]
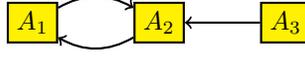
\begin{center}
\begin{tikzpicture}[->,thick]
\tikzstyle{every node}=[draw,rectangle,fill=yellow]
%%%background rectangle/.style={fill=blue!20},show background rectangle]
%shape=circle,fill=blue!20,draw
			\node (a1) [] {$A_1$};
			\node (a2) [right=of a1] {$A_2$};
			\node (a3) [right=of a2] {$A_3$};
			\path	(a1)[bend left] edge (a2);
			\path	(a2)[bend left] edge  (a1);
			\path	(a3) edge (a2);
			%\path	(a3) edge node[] {} (a2);
\end{tikzpicture}
\end{center}
\end{example}

Arguments can work together as a coalition by attacking other arguments and by defending their members from attack as follows.

\begin{definition}
\label{def:coalition}
Let $S \subseteq {\cal A}$ be a set of arguments.
\begin{itemize}

\item $S$ {\bf attacks} $A_j \in {\cal A}$ 
iff there is an argument $A_i \in S$ such that $A_i$ attacks $A_j$.

\item $S$ {\bf defends} $A_i \in S$ 
iff for each argument $A_j \in {\cal A}$, if $A_j$ attacks $A_i$ then $S$ attacks $A_j$.

\end{itemize}
\end{definition}

The following gives a requirement that should hold for a coalition of arguments to make sense. If it holds, it means that the arguments in the set offer a consistent view on the topic of the argument graph.

\begin{definition}
A set $S \subseteq {\cal A}$ of arguments is {\bf conflict-free} iff there are no arguments
$A_i$ and $A_j$ in $S$ such that $A_i$ attacks $A_j$.
\end{definition}

Now, we consider how we can find an acceptable set of arguments from an abstract argument graph.
The simplest case of arguments that can be accepted is as follows.

\begin{definition}
A set $S \subseteq {\cal A}$ of arguments is {\bf admissible} 
iff $S$ is conflict-free and defends all its arguments.
\end{definition}

The intuition here is that for a set of arguments to be accepted, 
we require that, if any one of them is challenged by a counterargument, 
then they offer grounds to challenge, in turn, the counterargument.
There always exists at least one admissible set:
The empty set is always admissible.

Clearly, the notion of admissible sets of arguments is the minimum requirement for a
set of arguments to be accepted. 
We will focus on the following classes of acceptable arguments.

\begin{definition}
\label{def:extensions}
Let $\Gamma$ be a conflict-free set of arguments,
and let ${\sf Defended}:\wp({\cal A}) \mapsto \wp({\cal A})$ be a function 
such that ${\sf Defended}(\Gamma) = \{ A \mid \Gamma \mbox{ defends } A \}$.
\begin{enumerate}

\item $\Gamma$ is a {\bf complete extension } 
iff $\Gamma = {\sf Defended}(\Gamma)$

\item $\Gamma$ is a {\bf grounded extension } 
iff it is the minimal (w.r.t. set inclusion) complete extension.

\item $\Gamma$ is a {\bf preferred extension } 
iff it is a maximal (w.r.t. set inclusion) complete extension.

\item $\Gamma$ is a {\bf stable extension }
iff it is a preferred extension that attacks every argument that is not in the extension.

\end{enumerate}
\end{definition}

The grounded extension is always unique, whereas there may be multiple preferred extensions.
We illustrate these definitions with the following examples. 
As can be seen from the examples, the grounded extension provides a skeptical view on which arguments can be accepted, whereas each preferred extension take a credulous view on which arguments can be accepted.

\begin{example}
\label{ex:extension1}
Continuing Example \ref{ex:intro1},
there is only one complete set, and so this is both grounded and preferred.
%The conflict free sets are $\{ \}$, $\{ A_1 \}$, $\{ A_2 \}$, $\{ A_3\}$, and $\{ A_1, A_3 \}$; 
%The admissible sets are $\{ \}$, $\{ A_1 \}$, $\{ A_3 \}$, and $\{ A_1, A_3 \}$;
%And the only complete set is $\{ A_1, A_3 \}$, and so this set is grounded and preferred.
Note, $\{ A_1, A_2 \}$, $\{ A_2, A_3 \}$, and $\{ A_1, A_2, A_3 \}$ are not conflictfree subsets. 
Only the conflictfree subsets are given in the table.
{\em
\begin{center}
\begin{tabular}{|c| c|c| c|c| c| c|}
\hline
& Conflict free & Admissible & Complete & Grounded & Preferred & Stable\\
%& & && & \\
\hline
\hline
$\{ \}$ 				& $\checkmark$ & $\checkmark$ & $\times$ & $\times$& $\times$&$\times$\\
$\{ A_1 \}$ 			& $\checkmark$ & $\checkmark$ &$\times$ & $\times$& $\times$&$\times$\\
$\{ A_2 \}$ 			& $\checkmark$ &$\times$ &$\times$ &$\times$ &$\times$ &$\times$\\
$\{ A_3 \}$ 			& $\checkmark$ & $\checkmark$ & $\times$& $\times$& $\times$&$\times$\\
%$\{ A_1, A_2 \}$ 		& & & & & &\\
$\{ A_1, A_3 \}$ 		& $\checkmark$ & $\checkmark$ & $\checkmark$ & $\checkmark$ & $\checkmark$ & $\checkmark$\\
%$\{ A_2, A_3 \}$ 		& & & & & &\\
%$\{ A_1, A_2, A_3 \}$ 	& & & & & &\\
\hline
\end{tabular}
\end{center}
}
\end{example}

\begin{example}
\label{ex:extension2}
Consider the following argument graph.
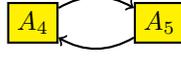
\begin{center}
\begin{tikzpicture}[->,thick]
\tikzstyle{every node}=[draw,rectangle,fill=yellow]
%%%background rectangle/.style={fill=blue!20},show background rectangle]
%shape=circle,fill=blue!20,draw
			\node (a1) [] {$A_4$};
			\node (a2) [right=of a1] {$A_5$};
			%\node (a3) [right=of a2] {$A_3$};
			\path	(a1)[bend left] edge (a2);
			\path	(a2)[bend left] edge  (a1);
			%\path	(a3) edge (a2);
			%\path	(a3) edge node[] {} (a2);
\end{tikzpicture}
\end{center}
%Consider the situation where we have 
%just two arguments $A_4$ and $A_5$ that attack each other. 
For this, there are two preferred sets, neither of which is grounded.
Note $\{ A_4, A_5 \}$ is not conflictfree.
Only the conflictfree subsets are given in the table.
{\em
\begin{center}
\begin{tabular}{|c|c|c| c|c| c|c| c|}
\hline
& Conflict free & Admissible & Complete & Grounded & Preferred & Stable \\
%& & && & \\
\hline
\hline
$\{ \}$ 				& $\checkmark$ & $\checkmark$ & $\checkmark$& $\checkmark$& $\times$ & $\times$\\
$\{ A_4 \}$ 			& $\checkmark$ & $\checkmark$ & $\checkmark$& $\times$& $\checkmark$ & $\checkmark$\\
$\{ A_5 \}$ 			& $\checkmark$ & $\checkmark$ & $\checkmark$& $\times$& $\checkmark$ & $\checkmark$\\
%%%%%$\{ A_4, A_5 \}$ 		& & & & & &\\
\hline
\end{tabular}
\end{center}
}
\end{example}

The formalization we have reviewed in this section is abstract 
because both the nature of
the arguments and the nature of the attack relation are ignored.
In particular, the internal (logical) structure of each of the arguments is not made explicit.
Nevertheless, Dung's proposal for abstract argumentation is ideal for clearly representing arguments and counterarguments, and for intuitively determining which arguments should be accepted (depending on whether we want to take a credulous or skeptical perspective).

%%%%%%We harness abstract argumentation in our general framework for aggregating evidence. We will introduce mechanisms for generating arguments from the evidence, and for generating the attacks relation based on the preferences over the arguments. In this way, we will instantiate abstract argumentation with logical arguments. This means that we can use Dung's definitions for determining which sets of arguments are acceptable, and thereby determine which aggregations of the evidence are acceptable.

Given an argument graph, let ${\sf Extensions}_{\sigma}(G)$ to denote the set of extensions according to $\sigma$ where 
$\sigma = \co$ denotes the complete extensions, 
$\sigma = \pr$ denotes the preferred extensions, 
$\sigma = \gr$ denotes the grounded extensions, 
and  $\sigma = \st$ denotes the stable extensions.

%%%%%%%%%%%%%%%%%%%%%%%%%%%%%%%%%%%%%%%%%%%%%%%%%%
\subsection{Labelling-based semantics}
%%%%%%%%%%%%%%%%%%%%%%%
 \label{section:labelsemantics}

We now review an alternative way of defining semantics for abstract argumentation proposed by Caminada and Gabbay \cite{Caminada09}. 
A labelling $L$ is a function $L:{\sf Nodes}(G)\rightarrow \{\argin,\argout,\argundec\}$ that assigns to each argument $A\in{\sf Nodes}(G)$ either the value $\argin$, meaning that the argument is accepted, $\argout$, meaning that the argument is not accepted, or $\argundec$, meaning that the status of the argument is undecided. Let $\argin(L)=\{A\mid L(A)=\argin\}$ and $\argout(L)$ resp.\ $\argundec(L)$ be defined analogously. The set $\argin(L)$ for a labelling $L$ is also called an \emph{extension}. A labelling $L$ is called \emph{conflict-free} if for no $A,B\in\argin(L)$ we have that $(A,B) \in {\sf Arcs}(G)$.

\begin{definition}
A labelling $L$ is called {\bf admissible} if and only if for all arguments $A\in{\sf Nodes}(G)$
\begin{enumerate}
	\item if $L(A)=\argout$ then there is $B\in{\sf Nodes}(G)$ with $L(B)=\argin$ and $(B,A) \in {\sf Arcs}(G)$, and
	\item if $L(A)=\argin$ then $L(B)=\argout$ for all $B\in{\sf Nodes}(G)$ with $(B,A) \in {\sf Arcs}(G)$,
\end{enumerate}
and it is called {\bf complete} if, additionally, it satisfies
\begin{enumerate}
	\setcounter{enumi}{2}
	\item if $L(A)=\argundec$ then there is no $B\in{\sf Nodes}(G)$ with $(B,A) \in {\sf Arcs}(G)$ and $L(B)=\argin$ and there is a $B'\in{\sf Nodes}(G)$ with $(B',A) \in {\sf Arcs}(G)$ and $L(B')\neq\argout$.
\end{enumerate}
\end{definition}

The intuition behind admissibility is that an argument can only be accepted if there are no attackers that are accepted and if an argument is not accepted then there has to be some reasonable grounds. The idea behind the completeness property is that the status of an argument is only $\argundec$ if it cannot be classified as $\argin$ or $\argout$. Different types of classical semantics can be obtained by imposing further constraints. 

\begin{definition}
Let $G$ be an argument graph, 
let $L:{\sf Nodes}(G)\rightarrow \{\argin,\argout,\argundec\}$ be a complete labelling,
and for all statements, minimality/maximality is with respect to set inclusion.
\begin{itemize}
	\item $L$ is {\bf grounded} if and only if $\argin(L)$ is minimal.
	\item $L$ is {\bf preferred} if and only if $\argin(L)$ is maximal.
	\item $L$ is {\bf stable} if and only if $\argundec(L)=\emptyset$.
%	\item $L$ is {\bf semi-stable} if and only if $\argundec(L)$ is minimal.
\end{itemize}
\end{definition}

\begin{example}
Continuing Example \ref{ex:extension1},
there is one complete labelling.
\begin{center}
\begin{tabular}{|c|c|c|c|l|}
\hline
& $A_1$ & $A_2$ & $A_3$ & Type\\
\hline
\hline
$L_1$ & $\argin$ & $\argout$ & $\argin$ & grounded, preferred, stable \\
\hline
\end{tabular}
\end{center}
\end{example}

\begin{example}
Continuing Example \ref{ex:extension2},
there are three complete labellings.
\begin{center}
\begin{tabular}{|c|c|c|l|}
\hline
& $A_3$ & $A_4$ & Type\\
\hline
\hline
$L_1$ & $\argundec$ & $\argundec$ & grounded \\
$L_2$ & $\argin$ & $\argout$ & preferred, stable \\
$L_3$ & $\argout$ & $\argin$ & preferred, stable \\
\hline
\end{tabular}
\end{center}
\end{example}

The extension-based semantics and labelling-based semantics are equivalent. So for instance, for the grounded extension $\Gamma$ for argument graph $G$, and the grounded labelling $G$, we have $\Gamma = \argin(L)$. Similarly, each preferred extension (respectively stable) is represented by a preferred (respectively stable) labelling.

%%%%%%%%%%%%%%%%%%%%%%%%%%%%%%%%%%%%%%%%%%%%%%%%%%
\subsection{Subsidiary definitions}
%%%%%%%%%%%%%%%%%%%%%%%%%%%%%%%%%%%

We now consider some further simple definitions that we will use as subsidiary functions for our measures of inconsistency for abstract argument graphs. 

\begin{definition}
For a graph $G$, and an argument $A$, 
the indegree and outdegree functions are defined as follows.
\begin{itemize}
\item ${\sf Indegree}(G,A) = |\{ (B,A) \mid (B,A)  \in {\sf Arcs}(G) \}$
\item ${\sf Outdegree}(G,A) = |\{ (A,B) \mid (A,B)  \in {\sf Arcs}(G) \}$
\end{itemize}
\end{definition}

\begin{example}
For the following graph,
${\sf Indegree}(G,A_1) = 0$, 
${\sf Indegree}(G,A_2) = 3$, 
${\sf Indegree}(G,A_3) = 2$, 
${\sf Indegree}(G,A_4) = 2$, 
${\sf Outdegree}(G,A_1) = 2$, 
${\sf Outdegree}(G,A_2) = 2$, 
${\sf Outdegree}(G,A_3) = 3$, 
and
${\sf Outdegree}(G,A_4) = 0$.
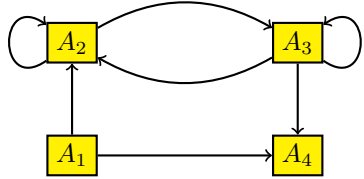
\begin{center}
\begin{tikzpicture}[->,thick]
\tikzstyle{every node}=[draw,rectangle,fill=yellow]
\node (A1)  at (-1.5,0) {$A_1$};
\node (A2)  at (-1.5,1.5) {$A_2$};
\node (A3) at (1.5,1.5) {$A_3$};
\node (A4) at (1.5,0) {$A_4$};
\path	(A1)[] edge[] (A2);		
\path	(A2)[] edge[bend left] (A3);		
\path	(A3)[] edge[] (A4);	
\path	(A1)[] edge[] (A4);		
\path	(A3)[] edge[bend left] (A2);		
%\path	(A2)[] edge[] (A1);	
%\path	(A4)[] edge[] (A1);	
%\draw (A1) .. controls (-0.8,-0.8) and (0.8,-0.8) .. (A1);
\draw (A2) .. controls (-2.5,0.8) and (-2.5,2.2) .. (A2);
\draw (A3) .. controls (2.5,0.8) and (2.5,2.2) .. (A3);			
\end{tikzpicture}
\end{center}
\end{example}

\begin{definition}
For graphs $G_1$ and $G_2$, $G_1$ is a {\bf subgraph} of $G_2$, denoted $G_1 \sqsubseteq G_2$, iff
\[
{\sf Nodes}(G_1) \subseteq {\sf Nodes}(G_2)
\mbox{ and } 
{\sf Arcs}(G_1) \subseteq ({\sf Arcs}(G_2) \cap ({\sf Nodes}(G_1) \times {\sf Nodes}(G_1)))
\]
\end{definition}

\begin{example}
For the following graphs $G_1$ (left) and $G_2$ (right),
$G_1 \sqsubseteq G_2$ holds. 
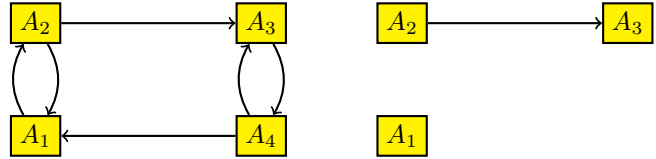
\begin{center}
\begin{tikzpicture}[->,thick]
\tikzstyle{every node}=[draw,rectangle,fill=yellow]
\node (A1)  at (-1.5,0) {$A_1$};
\node (A2)  at (-1.5,1.5) {$A_2$};
\node (A3) at (1.5,1.5) {$A_3$};
\node (A4) at (1.5,0) {$A_4$};
\path	(A1)[] edge[bend left] (A2);		
\path	(A2)[] edge[bend left] (A1);		
\path	(A2)[] edge[] (A3);		
\path	(A3)[] edge[bend left] (A4);	
\path	(A4)[] edge[bend left] (A3);	
\path	(A4)[] edge[] (A1);		
%\draw (A1) .. controls (-0.8,-0.8) and (0.8,-0.8) .. (A1);
%\draw (A2) .. controls (-2.5,0.8) and (-2.5,2.2) .. (A2);
%\draw (A3) .. controls (2.5,0.8) and (2.5,2.2) .. (A3);			
\end{tikzpicture}
\hspace{1cm}
\begin{tikzpicture}[->,thick]
\tikzstyle{every node}=[draw,rectangle,fill=yellow]
\node (A1)  at (-1.5,0) {$A_1$};
\node (A2)  at (-1.5,1.5) {$A_2$};
\node (A3) at (1.5,1.5) {$A_3$};
%\node (A4) at (1.5,0) {$A_4$};
%\path	(A1)[] edge[bend left] (A2);		
%\path	(A2)[] edge[bend left] (A1);		
\path	(A2)[] edge[] (A3);		
%\path	(A3)[] edge[bend left] (A4);	
%\path	(A4)[] edge[bend left] (A3);	
%\path	(A4)[] edge[] (A1);		
%\draw (A1) .. controls (-0.8,-0.8) and (0.8,-0.8) .. (A1);
%\draw (A2) .. controls (-2.5,0.8) and (-2.5,2.2) .. (A2);
%\draw (A3) .. controls (2.5,0.8) and (2.5,2.2) .. (A3);			
\end{tikzpicture}
\end{center}
\end{example}

\begin{definition}
For graph $G$, and a set of arguments $X \subseteq {\sf Nodes}(G)$, the {\bf induced graph} 
is a graph $G'$ such that $G' \sqsubseteq G$ and ${\sf Nodes}(G') = X$ and ${\sf Arcs}(G') = {\sf Arcs}(G) \cap (X \times X)$. Let  ${\sf Induced}(G,X)$ be the function that returns the induced graph for $G$ and $X$. 
\end{definition}

\begin{example}
For the graph $G$ (left), and $X = \{ A_2,A_3, A_4\}$,  the induced graph is ${\sf Induced}(G,X)$ (right)
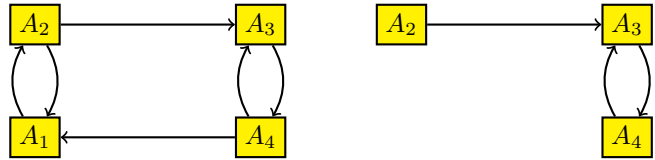
\begin{center}
\begin{tikzpicture}[->,thick]
\tikzstyle{every node}=[draw,rectangle,fill=yellow]
\node (A1)  at (-1.5,0) {$A_1$};
\node (A2)  at (-1.5,1.5) {$A_2$};
\node (A3) at (1.5,1.5) {$A_3$};
\node (A4) at (1.5,0) {$A_4$};
\path	(A1)[] edge[bend left] (A2);		
\path	(A2)[] edge[bend left] (A1);		
\path	(A2)[] edge[] (A3);		
\path	(A3)[] edge[bend left] (A4);	
\path	(A4)[] edge[bend left] (A3);	
\path	(A4)[] edge[] (A1);		
%\draw (A1) .. controls (-0.8,-0.8) and (0.8,-0.8) .. (A1);
%\draw (A2) .. controls (-2.5,0.8) and (-2.5,2.2) .. (A2);
%\draw (A3) .. controls (2.5,0.8) and (2.5,2.2) .. (A3);			
\end{tikzpicture}
\hspace{1cm}
\begin{tikzpicture}[->,thick]
\tikzstyle{every node}=[draw,rectangle,fill=yellow]
%\node (A1)  at (-1.5,0) {$A_1$};
\node (A2)  at (-1.5,1.5) {$A_2$};
\node (A3) at (1.5,1.5) {$A_3$};
\node (A4) at (1.5,0) {$A_4$};
%\path	(A1)[] edge[bend left] (A2);		
%\path	(A2)[] edge[bend left] (A1);		
\path	(A2)[] edge[] (A3);		
\path	(A3)[] edge[bend left] (A4);	
\path	(A4)[] edge[bend left] (A3);	
%\path	(A4)[] edge[] (A1);		
%\draw (A1) .. controls (-0.8,-0.8) and (0.8,-0.8) .. (A1);
%\draw (A2) .. controls (-2.5,0.8) and (-2.5,2.2) .. (A2);
%\draw (A3) .. controls (2.5,0.8) and (2.5,2.2) .. (A3);			
\end{tikzpicture}
\end{center}
\end{example}

\begin{definition}
For graphs $G_1$ and $G_2$, the {\bf composition} of $G_1$ and $G_2$, denoted $G_1+G_2$, is 
\[
({\sf Nodes}(G_1) \cup {\sf Nodes}(G_2)), ({\sf Arcs}(G_1) \cup {\sf Arcs}(G_2))
\]
\end{definition}

\begin{example}
For the graph $G_1$ (left), $G_2$ (middle), and $G_1 + G_2$ (right).
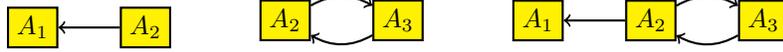
\begin{center}
\begin{tikzpicture}[->,thick]
\tikzstyle{every node}=[draw,rectangle,fill=yellow]
\node (A1)  at (0,0) {$A_1$};
\node (A2)  at (1.5,0) {$A_2$};
\path	(A2)[] edge[] (A1);		
\end{tikzpicture}
\hspace{1cm}
\begin{tikzpicture}[->,thick]
\tikzstyle{every node}=[draw,rectangle,fill=yellow]
\node (A2)  at (0,1.5) {$A_2$};
\node (A3) at (1.5,1.5) {$A_3$};
\path	(A3)[] edge[bend left] (A2);	
\path	(A2)[] edge[bend left] (A3);	
\end{tikzpicture}
\hspace{1cm}
\begin{tikzpicture}[->,thick]
\tikzstyle{every node}=[draw,rectangle,fill=yellow]
\node (A1)  at (0,0) {$A_1$};
\node (A2)  at (1.5,0) {$A_2$};
\node (A3) at (3,0) {$A_3$};
\path	(A2)[] edge[] (A1);		
\path	(A3)[] edge[bend left] (A2);	
\path	(A2)[] edge[bend left] (A3);	
\end{tikzpicture}
\end{center}
\end{example}

\begin{definition}
A graph $G$ is {\bf complete} iff for all $A,B \in {\sf Nodes}(G)$, $(A,B) \in {\sf Arcs}(G)$.
\end{definition}

\begin{example}
\label{ex:complete}
The following is a complete argument graph with three nodes.
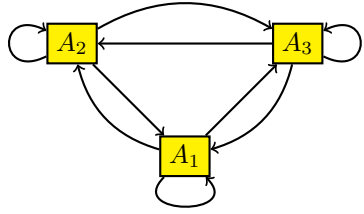
\begin{center}
\begin{tikzpicture}[->,thick]
\tikzstyle{every node}=[draw,rectangle,fill=yellow]
\node (A1)  at (0,0) {$A_1$};
\node (A2)  at (-1.5,1.5) {$A_2$};
\node (A3) at (1.5,1.5) {$A_3$};
%			\node (A4) at (1.5,0) {$A_4$};
\path	(A1)[] edge[bend left] (A2);		
\path	(A2)[] edge[bend left] (A3);		
\path	(A3)[] edge[bend left] (A1);	
\path	(A1)[] edge[] (A3);		
\path	(A3)[] edge[] (A2);		
\path	(A2)[] edge[] (A1);	
%			\path	(A4)[] edge[bend left] (A1);	
\draw (A1) .. controls (-0.8,-0.8) and (0.8,-0.8) .. (A1);
\draw (A2) .. controls (-2.5,1) and (-2.5,2) .. (A2);
\draw (A3) .. controls (2.5,1) and (2.5,2) .. (A3);			
\end{tikzpicture}
\end{center}
\end{example}

In the following definition, we define a cycle as a subset of nodes in the graph for which there is a circular path involving all these nodes. 

\begin{definition}
For a graph $G$, the {\bf cycles} in $G$ 
are ${\sf Cycles}(G) =$ 
\[
\{ \{A_1,\ldots,A_n\} \subseteq {\sf Nodes}(G) \mid 
\mbox{ for each } i \in \{1,\ldots,n-1\}, (A_i,A_{i+1}) \in {\sf Arcs}(G) \mbox{ and } A_1 = A_n \}
\]
\end{definition}

\begin{example}
In Example \ref{ex:complete}, there 3 cycles of length 1 (i.e. each of the nodes has a self-attack), 3 cycles of length 2 (i.e. each pair of nodes has a cycle), and 1 cycle of length 3 (i.e. there is a cycle involving all three nodes). 
\end{example}

\begin{definition}
A graph $G$ is {\bf disjoint} with graph $G'$ iff ${\sf Nodes}(G) \cap {\sf Nodes}(G') = \emptyset$.
\end{definition}

\begin{definition}
A graph $G$ has an {\bf inverse graph} defined as ${\sf Invert}(G) = ({\sf Nodes}(G),{\sf InverseArcs}(G))$ where 
\[
{\sf InverseArcs}(G) = \{ (B,A) \mid (A,B) \in {\sf Arcs}(G) \}
\]
\end{definition}

\begin{definition}
A graph $G$ is {\bf isomorphic} with graph $G'$ iff there is a bijection $f: {\sf Nodes}(G) \rightarrow {\sf Nodes}(G')$ such that $(A,B) \in {\sf Arcs}(G)$ iff $(f(A),f(B)) \in {\sf Arcs}(G)$.
\end{definition}

The following definition identifies a graph as being connected when for each pair of nodes in the graph, there is a sequence of arcs connecting them.  Equivalently, we say that a pair of nodes $A$ and $B$ is connected when there is an arc for $A$ to $B$, or an arc from $B$ to $A$, or there is a node $C$ such that $A$ is connected to $C$ and $C$ is connected to $B$. Then, a graph is connected when for each pair of nodes $A$ and $B$, $A$ is connected to $B$.

\begin{definition}
$G' \sqsubseteq G$ is {\bf connected} iff for all $A_i,A_j \in {\sf Nodes}(G')$ there is a path from $A_i$ to $A_j$ in ${\sf Arcs}(G^*)$ where ${\sf Arcs}(G^*) = \{ (A,B),(B,A) \mid (A,B) \in {\sf Arcs}(G) \}$.
\end{definition}

The following definition for a multi-node component is a variant of the usual definition for a component. It is a component that has at least two elements. 

\begin{definition}
$G'$ is a {\bf multi-node component} of $G$
iff (1) $G' \sqsubseteq G$
and (2) $G'$ is connected
and (3) there is no $G''$ s.t. $G' \sqsubseteq G''$ and $G'' \subseteq G$ and $G''$ is connected
and (4) $|{\sf Nodes}(G)| \geq 2$. 
Let ${\sf Components}(G)$ = $\{ G' \sqsubseteq G \mid G' \mbox{ is a multi-node component of } G \}$. 
\end{definition}

\begin{example}
In the following graph, there are three multi-node components (left, middle, right).
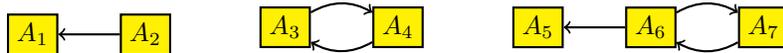
\begin{center}
\begin{tikzpicture}[->,thick]
\tikzstyle{every node}=[draw,rectangle,fill=yellow]
\node (A1)  at (0,0) {$A_1$};
\node (A2)  at (1.5,0) {$A_2$};
\path	(A2)[] edge[] (A1);		
\end{tikzpicture}
\hspace{1cm}
\begin{tikzpicture}[->,thick]
\tikzstyle{every node}=[draw,rectangle,fill=yellow]
\node (A2)  at (0,1.5) {$A_3$};
\node (A3) at (1.5,1.5) {$A_4$};
\path	(A3)[] edge[bend left] (A2);	
\path	(A2)[] edge[bend left] (A3);	
\end{tikzpicture}
\hspace{1cm}
\begin{tikzpicture}[->,thick]
\tikzstyle{every node}=[draw,rectangle,fill=yellow]
\node (A1)  at (0,0) {$A_5$};
\node (A2)  at (1.5,0) {$A_6$};
\node (A3) at (3,0) {$A_7$};
\path	(A2)[] edge[] (A1);		
\path	(A3)[] edge[bend left] (A2);	
\path	(A2)[] edge[bend left] (A3);	
\end{tikzpicture}
\end{center}\end{example}

We introduce the definition for multi-node components because it forms the basis of a potentially useful measure of inconsistency that we will introduce in the next section. 

%%%%%%%%%%%%%%%%%%%%%%%%%%%%%%%%%%%%%%%%%%%%%%%%%%
%%%%%%%%%%%%%%%%%%%%%%%%%%%%%%%%%%%%%%%%%%%%%%%%%%
%%%%%%%%%%%%%%%%%%%%%%%%%%%%%%%%%%%%%%%%%%%%%%%%%%
\section{Inconsistency measures for abstract argumentation}
%%%%%%%%%%%%%%%%%%%
\label{section:graphinc}

Following developments  in inconsistency measures for logical knowledgebases, we define a graph-based inconsistency measure as  a function that assigns a real number to each graph such that the following constraints of consistency and freeness are satisfied.
We explain these constraints as follows:
(Consistency) If a graph has no arcs, then the graph contains no counterarguments, and hence it is consistent;
and (Freeness) Adding an argument that does not attack any other arguments does not change the inconsistency of the graph.

\begin{definition}
\label{def:inconsistencymeasure}
A {\bf graph-based inconsistency measure} is a function $I:{\cal G} \rightarrow \reals$ such that 
\begin{enumerate}
\item (Consistency) $I(G) = 0$ if ${\sf Arcs}(G) = \emptyset$.
\item (Freeness) $I(G) = I(G')$ if ${\sf Nodes}(G) = {\sf Nodes}(G')\setminus\{A\}$ 
and ${\sf Arcs}(G) = {\sf Arcs}(G')$.
\end{enumerate}
\end{definition}

The following are further optional properties of a graph-based inconsistency measure: 
(Monotonicity) Adding arguments and counterarguments cannot decrease inconsistency;
(Inversion) If $G$ and $G'$ have the same arguments, but each attack in $G$ is reversed in $G'$, then they have the same degree of inconsistency; 
(Isomorphism) If $G$ and $G'$ have the same structure, then they have the same degree of inconsistency;
and (Disjoint additivity) If $G_1$ and $G_2$ are disjoint, then the inconsistency of $G_1+G_2$ is the sum of the inconsistency of $G_1$ and $G_2$.

\begin{definition}
\label{def:graphincproperties}
The following are further properties for a graph-based inconsistency measure.
\begin{itemize}
\item (Monotonicity) $I(G) \leq I(G')$ if $G \sqsubseteq G'$.
\item (Inversion) $I(G) = I(G')$ if $G' = {\sf Invert}(G)$.
\item (Isomorphic invariance) $I(G) = I(G')$ if $G$ and $G'$ are isomorphic . 
\item (Disjoint additivity) $I(G_1+G_2) = I(G_1) + I(G_2)$ if $G_1$ and $G_2$ are disjoint.
\item (Super-additivity) $I(G_1+G_2) \geq I(G_1) + I(G_2)$.
\end{itemize}
\end{definition}

In the following subsections, we consider two classes of inconsistency measure for abstract argumentation, namely graph structure measures, and graph extension measures.  We will compare and contrast them using the above properties.

%, and graph distance measures. 

%%%%%%%%%%%%%%%%%%%%%%%%%%%%%%%%%%%%%%%%%%%%%%%%%%
\subsection{Graph structure measures}
%%%%%%%%%%%%%%%%%%%%%%%%%%%%%%%%%%%%%

We now provide some proposals for measures. We give examples of them, and then show that they are graph-based inconsistency measures. 

\begin{definition}
\label{def:graphstructuremeasures}
The following are measures $I:{\cal G}  \rightarrow \reals$.
\begin{itemize}

\item (Drastic) 
\[
\mbox{If } {\sf Arcs}(G) \neq \emptyset, \mbox{ then } I_{dr}(G) = 1, \mbox{ otherwise} I_{dr}(G) = 0
\]

\item (InSum) 
\[
I_{in}(G) = \sum_{A \in {\sf Nodes}(G) } {\sf Indegree}(G,A)
\]

\item (WeightedInSum) 
\[
I_{win}(G) = \sum_{A \in {\sf Nodes}(G) \mbox{ s.t. } {\sf Indegree}(G,A) \geq 1} \frac{1}{{\sf Indegree}(G,A)}
\]

\item (WeightedOutSum) 
\[
I_{wou}(G) = \sum_{A \in {\sf Nodes}(G) \mbox{ s.t. } {\sf Outdegree}(G,A) \geq 1} \frac{1}{{\sf Outdegree}(G,A)}
\]

\item (CycleCount) 
\[
I_{cc}(G) = |{\sf Cycles}(G)|
\]

\item (WeightedCycleCount) 
\[
I_{wcc}(G) = \sum_{C \in {\sf Cycles}(G)} \frac{1}{|C|}
\]

\item (WeightedComponentCount) 
\[
\mbox{If } {\sf Components}(G) \neq \emptyset,
\mbox{ then } I_{ic}(G) =  \left(\sum_{X \in {\sf Components}(G)} (|X|-1)^2\right), 
\mbox{ otherwise } I_{ic}(G) = 0
\]

%%%%%%$I_{ic}(G) =  (\sum_{X \in {\sf Components}(G)} |X|^2) - 1$ when $|{\sf Components}(G)| \geq 1$ and $I_{ic}(G) = 0$ otherwise.

%%%%%%$I_{ic}(G) =  (2^{|{\sf Components}(G)|} - 1)$ when $|{\sf Components}(G)| \geq 1$ and $I_{ic}(G) = 0$ otherwise.\todo{Define component to to be a connected subgraph with two or more nodes}

\end{itemize}
\end{definition}

We explain these measures as follows:
(Drastic) If a graph has attacks (i.e. counterarguments), then the degree of inconsistency is 1, otherwise the degree of inconsistency is 0;
(InSum) This is the sum of the indegree for the nodes in the graph;
(WeightedInSum) This is the sum of the inverse of indegree for the nodes in the graph and so a node with a lower indegree has higher contribution to the inconsistency. 
(WeightedOutSum) This is the sum of the inverse of outdegree for the nodes in the graph and so a node with a lower outdegree has higher contribution to the inconsistency. 
(CycleCount) The is the number of the cycles in the graph;
(WeightedCycleSum) This is the sum of the inverse of the number cycles in the graph and so a shorter cycle has higher contribution to the inconsistency;
and
(WeightedComponentCount) This is the sum of the cardinality squared of each component in the graph and so a larger component has a higher contribution to the inconsistency.

%This is the sum of 2 to the power of the number components in the graph and so a larger component has higher contribution to the inconsistency.

\begin{example}
Consider the following graphs $G_1$ (left) and $G_2$ (right).
\begin{center}
\begin{tikzpicture}[->,thick]
\tikzstyle{every node}=[draw,rectangle,fill=yellow]
			\node (A1)  at (0,2) {$A_1$};
			\node (A2)  at (0,1) {$A_2$};
			\node (A3) at (0,0) {$A_3$};
			\node (B) at (1.5,1) {$B$};
			\path	(A1)[] edge[] (B);		
			\path	(A2)[] edge[] (B);		
			\path	(A3)[] edge[] (B);	

			\node (A1a)  at (6,2) {$A_1$};
			%\node (A2a)  at (0,1) {$A_2$};
			\node (A3a) at (6,0) {$A_3$};
			\node (Ba) at (7.5,1) {$B$};
			\path	(A1a)[] edge[] (Ba);		
			%\path	(A2a)[] edge[] (Ba);		
			\path	(A3a)[] edge[] (Ba);	

\end{tikzpicture}
\end{center}
Hence, we get the following inconsistency evaluations.
\begin{center}
\begin{tabular}{|c|cccccc|}
\hline
& $I_{in}$ & $I_{win}$ & $I_{wou}$ & $I_{cc}$ & $I_{wcc}$ & $I_{ic}$\\
\hline
$G_1$ & 3& 1/3 & 3& 0&0&9\\
\hline
$G_2$ & 2& 1/2 & 2&0&0&4\\
\hline
\end{tabular}
\end{center}
\end{example}

\begin{example}
Consider the following graphs $G_1$ (left) and $G_2$ (right).
\begin{center}
\begin{tikzpicture}[->,thick]
\tikzstyle{every node}=[draw,rectangle,fill=yellow]
			\node (A1)  at (0,0) {$A_1$};
			\node (A2)  at (0,1.5) {$A_2$};
			\node (A3) at (1.5,1.5) {$A_3$};
			\node (A4) at (1.5,0) {$A_4$};
			\path	(A1)[] edge[bend left] (A2);		
			\path	(A2)[] edge[bend left] (A3);		
			\path	(A3)[] edge[bend left] (A4);	
			\path	(A4)[] edge[bend left] (A1);	
			
			\node (A1)  at (4,0) {$A_1$};
			\node (A2)  at (4.75,1.5) {$A_2$};
			\node (A3) at (5.5,0) {$A_3$};
			\path	(A1)[] edge[bend left] (A2);		
			\path	(A2)[] edge[bend left] (A3);		
			\path	(A3)[] edge[bend left] (A1);	
			
\end{tikzpicture}
\end{center}
Hence, we get the following inconsistency evaluations.
\begin{center}
\begin{tabular}{|c|cccccc|}
\hline
& $I_{in}$ & $I_{win}$ & $I_{wou}$ & $I_{cc}$ & $I_{wcc}$ & $I_{ic}$\\
\hline
$G_1$ & 4& 4 & 4& 1&1/4&9\\
\hline
$G_2$ & 3& 3 & 3&1&1/3&4\\
\hline
\end{tabular}
\end{center}
\end{example}

The graph structure measures are graph-based inconsistency measures as shown in the following result. 

\begin{proposition}
The $I_{dr}$, $I_{in}$, $I_{win}$, $I_{wou}$, $I_{cc}$, $I_{wcc}$ 
and $I_{ic}$ measures are graph-based inconsistency measures according 
to Definition \ref{def:inconsistencymeasure}.
\end{proposition}

\begin{proof}
We show satisfaction of consistency by assuming that ${\sf Arcs}(G) = \emptyset$:
($I_{dr}$) By definition, ${\sf Arcs}(G) = \emptyset$ implies $I_{dr}(G) = 0$; 
($I_{in}$) Since ${\sf Arcs}(G) = \emptyset$, there is no $A \in {\sf Nodes}(G)$ such that ${\sf Indegree}(G,A) > 0$, and so $I_{in}(G) = 0$; 
($I_{win}$) Ditto; 
($I_{wou}$) Ditto; 
($I_{cc}$) Since ${\sf Arcs}(G) = \emptyset$, ${\sf Cycles}(G) = \emptyset$, and so  $I_{cc}(G) = 0$; 
($I_{wcc}$) Ditto; 
($I_{ic}$)  Since ${\sf Arcs}(G) = \emptyset$, ${\sf Components}(G) = \emptyset$, and so  $I_{ic}(G) = 0$.
 We show satisfaction of freeness by assuming that ${\sf Nodes}(G) = {\sf Nodes}(G')\setminus\{A\}$ 
and ${\sf Arcs}(G) = {\sf Arcs}(G')$:
($I_{dr}$) Since ${\sf Arcs}(G) = {\sf Arcs}(G')$, 
if ${\sf Arcs}(G) = \emptyset$, $I_{dr}(G) = I_{dr}(G') = 0$
and if ${\sf Arcs}(G) \neq \emptyset$, $I_{dr}(G) = I_{dr}(G') = 1$; 
($I_{in}$) Since ${\sf Arcs}(G) = {\sf Arcs}(G')$,
${\sf Indegree}(G,A) = {\sf Indegree}(G',A)$ for all $A \in {\sf Nodes}(G)$, 
and ${\sf Indegree}(G',A) = 0$ for $A \in {\sf Nodes}(G')\setminus {\sf Nodes}(G)$, 
and so $I_{in}(G) = I_{in}(G')$; 
($I_{win}$) Ditto; 
($I_{wou}$) Ditto; 
($I_{cc}$) Since ${\sf Arcs}(G) = {\sf Arcs}(G')$, ${\sf Cycles}(G) = {\sf Cycles}(G')$, 
and so  $I_{cc}(G) = I_{cc}(G')$; 
($I_{wcc}$) Ditto; 
($I_{ic}$)  Since ${\sf Arcs}(G) = {\sf Arcs}(G')$, 
${\sf Components}(G) ={\sf Components}(G)$, and so  $I_{ic}(G) = I_{ic}(G')$.
\end{proof}

The following result gives some idea of the difference of scale for each measure. 

\begin{proposition}
\label{prop:complete}
If $G$ is a complete graph and $|{\sf Nodes}(G)| = n$, then
\[
\begin{array}{ccccccc}
I_{dr} = 1
& I_{in} = n^2
& I_{win} = 1
& I_{wou} = 1
& I_{cc} = 2^n - 1
& I_{wcc}  = \sum^n_{i = 1} [\frac{1}{i}  \times \frac{n!}{i!(n-i)!}]
& I_{ic} = (n-1)^2
\end{array}
\]
\end{proposition}

\begin{proof}
Assume $G$ is a complete graph and $|{\sf Nodes}(G)| = n$. 
($I_{dr}$) Since $G$ is complete, ${\sf Arcs}(G) \neq \emptyset$, 
and so $I_{dr} = 1$.
($I_{in}$) Since $G$ is complete, then for every node $A$, ${\sf Indegree}(G,A) = n$, 
and so $I_{in} = n^2$.
($I_{win}$) Since $G$ is complete, then for every node $A$, ${\sf Indegree}(G,A) = n$, 
and so  $I_{win}(G)$ = $\sum_{A \in {\sf Nodes}(G)} \frac{1}{{\sf Indegree}(G,A)}$ 
= $\sum_{A \in {\sf Nodes}(G)} \frac{1}{n}$
= $n \times \frac{1}{n}$
= 1.
($I_{wou}$) Ditto. 
($I_{cc}$) Since $G$ is complete, there are $2^n - 1$ subsets of nodes where each subset constitutes a cycle, and so $I_{cc} = 2^n - 1$.
($I_{wcc}$) Since $G$ is complete, for each non-empty subset of nodes constitute a cycle. 
Furthermore, for each $i \in \{1,\ldots,n\}$, there are $\frac{n!}{i!(n-i)!}$ subsets of cardinality $i$, 
and each of these contribute $\frac{1}{i}$ to the sum. 
Hence, the sum is $\sum^n_{i=1} [\frac{1}{i} \times \frac{n!}{i!(n-i)!}]$.
($I_{ic}$) Since $G$ is complete, there is a single component, and so $I_{ic} = (n-1)^2$.
\end{proof}

The following result shows which optional properties are satisfied by which measures.

\begin{proposition}
\label{prop:properties:structure}
For the graph structure measures in Definition \ref{def:graphstructuremeasures},
the adherence to the properties in Definition \ref{def:graphincproperties} is summarized in the following table.
\begin{center}
\begin{tabular}{|l|c|c|c|c|c|c|c|}
\hline
& $I_{dr}$ & $I_{in}$ & $I_{win}$ & $I_{wou}$ & $I_{cc}$ & $I_{wcc}$ & $I_{ic}$\\
\hline
\hline
Monotonicity & $\checkmark$& $\checkmark$ & $\checkmark$ & $\checkmark$ &$\checkmark$& $\checkmark$& $\checkmark$\\
\hline
Inversion & $\checkmark$& $\checkmark$ & $\times$&$\times$&$\checkmark$&$\checkmark$&$\checkmark$\\
\hline
Isomorphic invariance & $\checkmark$& $\checkmark$ & $\checkmark$&$\checkmark$&$\checkmark$&$\checkmark$&$\checkmark$\\
\hline
Disjoint additivity & $\times$& $\checkmark$ & $\checkmark$&$\checkmark$&$\checkmark$&$\checkmark$&$\checkmark$\\
\hline
Super-additivity & $\times$& $\checkmark$ & $\times$& $\times$&$\checkmark$&$\checkmark$&$\times$\\
\hline
\end{tabular}
\end{center}
\end{proposition}

\begin{proof}
We consider each property as follows.
\begin{itemize}
\item Monotonicity. 
($I_{dr}$, $I_{in}$, $I_{win}$, $I_{wou}$, $I_{cc}$, $I_{wcc}$,$I_{ic}$) 
%Follows directly from definition. 
%($I_{ic}$) Consider $G_1$ (left) and $G_2$ (right) where $I_{ic}(G_1) = 2$, $I_{ic}(G_2) = 1$, and $G_1 \sqsubseteq G_2$. 
%\begin{center}
%\begin{tikzpicture}[->,thick]
%\tikzstyle{every node}=[draw,rectangle,fill=yellow]
%\node (A1)  at (0,0) {$A_1$};
%\node (A2)  at (1.5,0) {$A_2$};
%\node (A3) at (3,0) {$A_3$};
%\node (A4) at (4.5,0) {$A_4$};
%\path	(A2)[] edge[] (A1);		
%\path	(A4)[] edge[] (A3);		
%\end{tikzpicture}
%\hspace{1cm}
%\begin{tikzpicture}[->,thick]
%\tikzstyle{every node}=[draw,rectangle,fill=yellow]
%\node (A1)  at (0,0) {$A_1$};
%\node (A2)  at (1.5,0) {$A_2$};
%\node (A3) at (3,0) {$A_3$};
%\node (A4) at (4.5,0) {$A_4$};
%\path	(A2)[] edge[] (A1);		
%\path	(A3)[] edge[] (A2);		
%\path	(A4)[] edge[] (A3);		
%\end{tikzpicture}
%\end{center}
\item Inversion. 
($I_{dr}$, $I_{in}$) Follows directly from definition.
($I_{win}$, $I_{wou}$) Consider $G$ (left) and ${\sf Invert}(G)$ (right),
where $I_{win}(G) = 1/2$, $I_{win}({\sf Invert}(G)) = 2$, $I_{wou}(G) = 2$,  
and $I_{wou}({\sf Invert}(G)) = 1/2$.
\begin{center}
\begin{tikzpicture}[->,thick]
\tikzstyle{every node}=[draw,rectangle,fill=yellow]
\node (A1)  at (0,0) {$A_1$};
\node (A2)  at (1.5,0) {$A_2$};
\node (A3) at (3,0) {$A_3$};
\path	(A1)[] edge[] (A2);		
\path	(A3)[] edge[] (A2);		
\end{tikzpicture}
\hspace{1cm}
\begin{tikzpicture}[->,thick]
\tikzstyle{every node}=[draw,rectangle,fill=yellow]
\node (A1)  at (0,0) {$A_1$};
\node (A2)  at (1.5,0) {$A_2$};
\node (A3) at (3,0) {$A_3$};
\path	(A2)[] edge[] (A1);		
\path	(A2)[] edge[] (A3);		
\end{tikzpicture}
\end{center}
( $I_{cc}$, $I_{wcc}$, $I_{ic}$) Follows directly from definition.
\item Isomorphic invariance. 
($I_{dr}$, $I_{in}$, $I_{win}$, $I_{wou}$, $I_{cc}$, $I_{wcc}$, $I_{ic}$) Follows directly from definition. 
\item Disjoint additivity. 
($I_{dr}$) Consider the following graphs $G_1$ (left) and $G_2$ (right)
where $I_{dr}(G_1) = 1$  and $I_{dr}(G_2) = 1$
but $I_{dr}(G_1+G_2) = 1$.
\begin{center}
\begin{tikzpicture}[->,thick]
\tikzstyle{every node}=[draw,rectangle,fill=yellow]
\node (A1)  at (0,0) {$A_1$};
\node (A2)  at (1.5,0) {$A_2$};
%\node (A3) at (3,0) {$A_3$};
\path	(A2)[] edge[] (A1);		
%\path	(A3)[] edge[] (A2);		
\end{tikzpicture}
\hspace{1cm}
\begin{tikzpicture}[->,thick]
\tikzstyle{every node}=[draw,rectangle,fill=yellow]
\node (A3)  at (0,0) {$A_3$};
\node (A4)  at (1.5,0) {$A_4$};
\path	(A4)[] edge[] (A3);		
\end{tikzpicture}
\end{center}
($I_{in}$, $I_{win}$, $I_{wou}$, $I_{cc}$, $I_{wcc}$, $I_{ic}$) Follows directly from definition.
\item Super-additivity. 
($I_{dr}$) See counterexample for disjoint additivity. 
($I_{in}$) Follows directly from definition.
($I_{win}$) Consider the following graphs $G_1$ (left) and $G_2$ (right)
where $I_{win}(G_1) = 1$  and $I_{win}(G_2) = 1$
but $I_{win}(G_1+G_2) = 1/2$.
\begin{center}
\begin{tikzpicture}[->,thick]
\tikzstyle{every node}=[draw,rectangle,fill=yellow]
\node (A1)  at (0,0) {$A_1$};
\node (A2)  at (1.5,0) {$A_2$};
%\node (A3) at (3,0) {$A_3$};
\path	(A2)[] edge[] (A1);		
%\path	(A3)[] edge[] (A2);		
\end{tikzpicture}
\hspace{1cm}
\begin{tikzpicture}[->,thick]
\tikzstyle{every node}=[draw,rectangle,fill=yellow]
\node (A1)  at (0,0) {$A_1$};
\node (A3)  at (1.5,0) {$A_3$};
\path	(A3)[] edge[] (A1);		
\end{tikzpicture}
\end{center}
($I_{wou}$) Use a similar counterexample to $I_{win}$.
($I_{cc}$, $I_{wcc}$) Follows directly from definition.
($I_{ic}$) Consider the following graphs $G_1$ (left) and $G_2$ (right)
where $I_{ic}(G_1) = 9$  and $I_{ic}(G_2) = 9$
but $I_{ic}(G_1+G_2) = 16$.
\begin{center}
\begin{tikzpicture}[->,thick]
\tikzstyle{every node}=[draw,rectangle,fill=yellow]
\node (A0)  at (-1.5,0) {$A_0$};
\node (A1)  at (0,0) {$A_1$};
\node (A2)  at (1.5,0) {$A_2$};
%\node (A3) at (3,0) {$A_3$};
\node (A4) at (3,0) {$A_4$};
\path	(A1)[] edge[] (A0);		
\path	(A2)[] edge[] (A1);		
\path	(A4)[] edge[] (A2);		
\end{tikzpicture}
\hspace{1cm}
\begin{tikzpicture}[->,thick]
\tikzstyle{every node}=[draw,rectangle,fill=yellow]
\node (A0)  at (-1.5,0) {$A_0$};
\node (A1)  at (0,0) {$A_1$};
\node (A2)  at (1.5,0) {$A_2$};
\node (A3)  at (3,0) {$A_3$};
\path	(A1)[] edge[] (A0);		
\path	(A2)[] edge[] (A1);		
\path	(A3)[] edge[] (A2);		
\end{tikzpicture}
\end{center}
\end{itemize}
\end{proof}

The following property of order-compatibility holds when two measures give the same ranking to all the graphs. If this holds, then there is some overlap in what the two measures offer. 

\begin{definition}
Measures $I_x$ and $I_y$ are {\bf order-compatible} if for all $G_1$ and $G_2$,
\[
I_x(K_1) < I_x(K_2) \mbox{ iff } I_y(K_1) < I_y(K_2)
\]
otherwise $I_x$ and $I_y$ are {\bf order-incompatible}. 
\end{definition}

\begin{proposition}
The $I_{dr}$, $I_{in}$, $I_{win}$, $I_{wou}$, $I_{cc}$, $I_{wcc}$ 
and $I_{ic}$ measures are pairwise incompatible.
\end{proposition}

\begin{proof}
From the differences in satisfaction of properties in Proposition \ref{prop:properties:structure}, 
$I_{dr}$ is pairwise incompatible with other measures,
$I_{ic}$ is pairwise incompatible with other measures,
and $I_{win}$, and $I_{wou}$ is pairwise incompatible with other measures,
though from the properties in Proposition \ref{prop:properties:structure}, we cannot discriminate $I_{win}$ from $I_{wou}$.
However, we can discriminate $I_{win}$ from $I_{wou}$ with graph $G_1$ (left) and $G_2$ (right),
where $I_{win}(G_1) = 1/2$, $I_{win}(G_2) = 2$, $I_{wou}(G_1) = 2$, and $I_{wou}(G_2) = 1/2$.
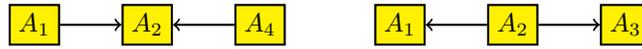
\begin{center}
\begin{tikzpicture}[->,thick]
\tikzstyle{every node}=[draw,rectangle,fill=yellow]
\node (A1)  at (0,0) {$A_1$};
\node (A2)  at (1.5,0) {$A_2$};
%\node (A3) at (3,0) {$A_3$};
\node (A3) at (3,0) {$A_4$};
\path	(A1)[] edge[] (A2);		
\path	(A3)[] edge[] (A2);		
\end{tikzpicture}
\hspace{1cm}
\begin{tikzpicture}[->,thick]
\tikzstyle{every node}=[draw,rectangle,fill=yellow]
\node (A1)  at (0,0) {$A_1$};
\node (A2)  at (1.5,0) {$A_2$};
\node (A3)  at (3,0) {$A_3$};
\path	(A2)[] edge[] (A1);		
\path	(A2)[] edge[] (A3);		
\end{tikzpicture}
\end{center}
Finally, from the properties in Proposition \ref{prop:properties:structure}, we cannot discriminate between $I_{in}$, $I_{wcc}$, $I_{cc}$.
However, from Proposition \ref{prop:complete}, assuming $G$ is complete, if $|{\sf Nodes}(G)| \in \{2,3,4\}$, then $I_{in}(G) > I_{cc}(G)$, 
whereas if $|{\sf Nodes}(G)| \geq 5$, then $I_{in}(G) < I_{cc}(G)$.
Similarly, if $|{\sf Nodes}(G)| = 2$, then $I_{in}(G) > I_{wcc}(G)$, 
whereas if $|{\sf Nodes}(G)| = 10$, then $I_{in}(G) < I_{wcc}(G)$,
and if $|{\sf Nodes}(G)| = 2$, then $I_{cc}(G) > I_{wcc}(G)$, 
whereas if $|{\sf Nodes}(G)| = 12$, then $I_{cc}(G) < I_{cc}(G)$
\end{proof}

In this section, we have considered a range of measures that take the structure of the argument graph into account. Each has it rationale, and any combination of them may provide useful insights into the nature of the conflict in an argumentation scenario.

%%%%%%%%%%%%%%%%%%%%%%%%%%%%%%%%%
%
%OLD STUFF ON USING EPISTEMIC PROBABILITIES
%
%\begin{definition}
%A belief assignment is a probability function $P: \wp({\sf Nodes}(G)) \rightarrow [0,1]$.
%\end{definition}
%
%\todo[inline]{Copy in stuff on epistemic probabilities from ecai including rational constraint} 
%
%\begin{definition}
%$R_P(G) = (N',R')$ where
%\begin{itemize}
%\item $N'$ = $\{ A \in {\sf Nodes}(G) \mid P(A) > 0.5 \}$
%\item $R'$ = $\{ (A,B) \in {\sf Arcs}(G) \mid P(A,B) > 0.5 \}$
%\end{itemize}
%\end{definition}
%
%\begin{definition}
%A belief assignment $P_{i+1}$ is {\bf more committed} than belief assignment $P_i$ iff
%\[
%I(R_{P_{i+1}}(G) \leq I(R_{P_i}(G))
%\]
%\end{definition}
%
%\begin{proposition}
%For all $G$, there is a $P$, such that $I(R_{P})(G) = 0$.
%\end{proposition}
%
%%%%%%%%%%%%%%%%%%%%%%%%%%%%%%%%%

%%%%%%%%%%%%%%%%%%%%%%%%%%%%%%%%%%%%%%%%%%%%%%%%%%
\subsection{Graph extension measures}
%%%%%%%%%%%%%%%

We now consider measures of inconsistency that take the extensions of the graph into account. In the following, we consider three measures: (1) $I_{pr}$ which gives the number of preferred extensions;
(2) $I_{ngr}$ which gives the number of arguments not in the grounded extension and not attacked by a member of the grounded extension;
and 
(3) $I_{nst}$ which gives the minimum number of arguments to be removed to get a stable extension.

\begin{definition}
The following are measures $I:{\cal G}  \rightarrow \reals$.
\begin{itemize}
\item (PreferredCount) 
\[
I_{pr} = {\sf Extensions}_{\pr}(G) - 1
\]
\item (NonGroundedCount) 
\[
I_{ngr} = | {\sf Nodes}(G) \setminus ({\sf Extensions}_{\gr}(G) \cup {\sf Attackees}(G)) |
\]
where  ${\sf Attackees}(G) = \{B \mid (A,B) \in {\sf Arcs}(G) \mbox{ and } A \in {\sf Extensions}_{\gr}(G)\}$.
\item (UnstableCount) 
\[
I_{ust} = {\sf min} \{ |X| \mid {\sf Extensions}_{\st}({\sf Induced}(G,X)) \mbox{ s.t. } X \subseteq {\sf Nodes}(G) \}
\]
\end{itemize}
\end{definition}

\begin{example}
For the following argument graph, 
${\sf Extensions}_{\pr}(G) = \{   \{A_4,A_6,A_8 \},  \{A_5,A_6,A_8 \} \}$, 
${\sf Extensions}_{\gr}(G) = \{   \{A_6,A_8 \} \}$, 
and 
$\{A_1,A_2,A_3\}$ is the minimum number of arguments to be removed to get a stable extension.
Hence, $I_{pr}(G) =  2$,  $I_{ngr}(G) =  5$, and $I_{ust}(G) =  3$.
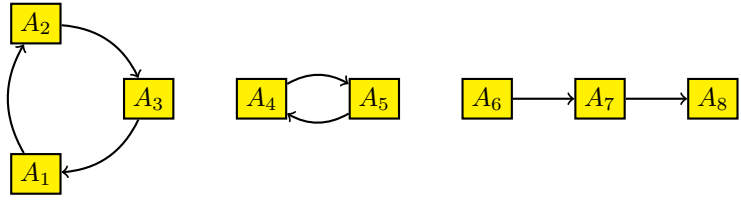
\begin{center}
\begin{tikzpicture}[->,thick]
\tikzstyle{every node}=[draw,rectangle,fill=yellow]
\node (A1)  at (0,0) {$A_1$};
\node (A2)  at (0,2) {$A_2$};
\node (A3) at (1.5,1) {$A_3$};
\node (A4) at (3,1) {$A_4$};
\node (A5) at (4.5,1) {$A_5$};
\node (A6) at (6,1) {$A_6$};
\node (A7) at (7.5,1) {$A_7$};
\node (A8) at (9,1) {$A_8$};
\path	(A1)[] edge[bend left] (A2);		
\path	(A2)[] edge[bend left] (A3);		
\path	(A3)[] edge[bend left] (A1);	
%\path	(A3)[] edge[bend left] (A4);		
%\path	(A4)[] edge[bend left] (A3);		
\path	(A4)[] edge[bend left] (A5);		
\path	(A5)[] edge[bend left] (A4);	
\path	(A6)[] edge[] (A7);	
\path	(A7)[] edge[] (A8);	
\end{tikzpicture}
\end{center}
\end{example}

The following result shows that the extension-based measures are indeed inconsistency measures but the subsequent result shows that most of the optional properties do no hold for these measures. 

\begin{proposition}
The $I_{pr}$, $I_{ngr}$, and $I_{ust}$ measures are graph-based inconsistency measures according 
to Definition \ref{def:inconsistencymeasure}.
\end{proposition}

\begin{proof}
We show satisfaction of consistency by assuming that ${\sf Arcs}(G) = \emptyset$:
($I_{pr}$) Since ${\sf Arcs}(G) = \emptyset$, 
therefore ${\sf Extensions}_{\pr}(G) = \{{\sf Nodes}(G)\}$,
and so $I_{pr}(G) = 0$; 
($I_{ngr}$) Since ${\sf Arcs}(G) = \emptyset$, 
therefore ${\sf Extensions}_{\gr}(G) = \{{\sf Nodes}(G)\}$,
and so $I_{ngr}(G) = 0$;  
($I_{ust}$) Since ${\sf Arcs}(G) = \emptyset$,
hence ${\sf Extensions}_{\st}(G) \neq \emptyset$,
and so the smallest $X$ such that ${\sf Extensions}_{\st}({\sf Induced}(G,X)) \neq \emptyset$
is $X = \emptyset$, and so $I_{ust}(G) = 0$.
We show satisfaction of freeness by assuming that ${\sf Nodes}(G) = {\sf Nodes}(G')\setminus\{A\}$ 
and ${\sf Arcs}(G) = {\sf Arcs}(G')$:
($I_{pr}$) From the assumption, $|{\sf Extensions}_{\gr}(G)| = |{\sf Extensions}_{\gr}(G')|$,
and therefore $I_{pr}(G) = I_{pr}(G')$;
($I_{ngr}$) Let $X$ be ${\sf Nodes}(G')\setminus {\sf Nodes}(G)$,
${\sf Extensions}_{\pr}(G) = \{ E \}$, 
and ${\sf Extensions}_{\pr}(G') = \{ E' \}$, 
and so from the assumptions,   $E = E'\setminus X$,
and hence ${\sf Nodes}(G) \setminus (E \cup {\sf Attackees}(G) )$ 
= $ {\sf Nodes}(G') \setminus (E' \cup {\sf Attackees}(G') ) $,
and so $I_{ngr}(G) = I_{ngr}(G')$;
($I_{ust}$)  Let $X_1$ be the smallest subset of nodes 
such that ${\sf Extensions}_{\st}({\sf Induced}(G,X_1)) \neq \emptyset$.
From the assumptions, $X_1$ is also the smallest subset of nodes 
such that ${\sf Extensions}_{\st}({\sf Induced}(G',X_1)) \neq \emptyset$.
Hence, $I_{ust}(G) = I_{ust}(G')$. 
\end{proof}

\begin{proposition}
\label{prop:properties:extension}
For the graph structure measures in Definition \ref{def:graphstructuremeasures},
the adherence to the properties in \ref{def:graphincproperties} is summarized in the following table.
\begin{center}
\begin{tabular}{|l|c|c|c|}
\hline
& $I_{pr}$ & $I_{ngr}$ & $I_{ust}$ \\
\hline
\hline
Monotonicity & $\times$& $\times$ & $\times$\\
\hline
Inversion & $\times$& $\times$ & $\times$\\
\hline
Isomorphic invariance & $\checkmark$& $\checkmark$  &  $\checkmark$\\
\hline
Disjoint additivity & $\times$& $\checkmark$ & $\checkmark$\\
\hline
Super-additivity & $\times$&$\times$ &$\times$ \\
\hline
\end{tabular}
\end{center}
\end{proposition}

\begin{proof}
We consider each property as follows.
\begin{itemize}
\item Monotonicity. 
($I_{pr}$) Consider $G_1$ (left) and $G_2$ (right) where $I_{pr}(G_1) = 1$, $I_{pr}(G_2) = 0$, and $G_1 \sqsubseteq G_2$. 
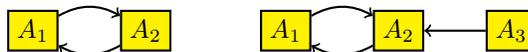
\begin{center}
\begin{tikzpicture}[->,thick]
\tikzstyle{every node}=[draw,rectangle,fill=yellow]
\node (A1)  at (0,0) {$A_1$};
\node (A2)  at (1.5,0) {$A_2$};
\path	(A2)[] edge[bend left] (A1);		
\path	(A1)[] edge[bend left] (A2);		
\end{tikzpicture}
\hspace{1cm}
\begin{tikzpicture}[->,thick]
\tikzstyle{every node}=[draw,rectangle,fill=yellow]
\node (A1)  at (0,0) {$A_1$};
\node (A2)  at (1.5,0) {$A_2$};
\node (A3) at (3,0) {$A_3$};
\path	(A2)[] edge[bend left] (A1);		
\path	(A1)[] edge[bend left] (A2);		
\path	(A3)[] edge[] (A2);		
\end{tikzpicture}
\end{center}

($I_{ngr}$) Consider $G_1$ (above left) and $G_2$ (above right) where $I_{ngr}(G_1) = 2$, $I_{ngr}(G_2) = 0$, and $G_1 \sqsubseteq G_2$. 
%Consider $G_1$ (left) and $G_2$ (right) where $I_{ngr}(G_1) = 1$, $I_{ngr}(G_2) = 0$, and $G_1 \sqsubseteq G_2$. 
%\begin{center}
%\begin{tikzpicture}[->,thick]
%\tikzstyle{every node}=[draw,rectangle,fill=yellow]
%\node (A1)  at (0,0) {$A_1$};
%\node (A2)  at (1.5,0) {$A_2$};
%\path	(A2)[] edge[] (A1);		
%\end{tikzpicture}
%\hspace{1cm}
%\begin{tikzpicture}[->,thick]
%\tikzstyle{every node}=[draw,rectangle,fill=yellow]
%\node (A1)  at (0,0) {$A_1$};
%\node (A2)  at (1.5,0) {$A_2$};
%\path	(A2)[] edge[bend left] (A1);		
%\path	(A1)[] edge[bend left] (A2);		
%\end{tikzpicture}
%\end{center}
($I_{ust}$) 
 Consider $G_1$ (left) and $G_2$ (right) where $I_{ust}(G_1) = 3$, $I_{ust}(G_2) = 0$, and $G_1 \sqsubseteq G_2$. 
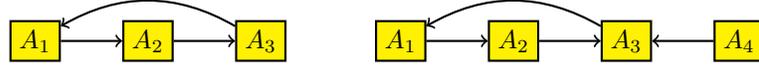
\begin{center}
\begin{tikzpicture}[->,thick]
\tikzstyle{every node}=[draw,rectangle,fill=yellow]
\node (A1)  at (0,0) {$A_1$};
\node (A2)  at (1.5,0) {$A_2$};
\node (A3)  at (3,0) {$A_3$};
\path	(A3)[] edge[bend right] (A1);		
\path	(A2)[] edge[] (A3);		
\path	(A1)[] edge[] (A2);		
\end{tikzpicture}
\hspace{1cm}
\begin{tikzpicture}[->,thick]
\tikzstyle{every node}=[draw,rectangle,fill=yellow]
\node (A1)  at (0,0) {$A_1$};
\node (A2)  at (1.5,0) {$A_2$};
\node (A3)  at (3,0) {$A_3$};
\node (A4)  at (4.5,0) {$A_4$};
\path	(A3)[] edge[bend right] (A1);		
\path	(A2)[] edge[] (A3);		
\path	(A1)[] edge[] (A2);		
\path	(A4)[] edge[] (A3);		
\end{tikzpicture}
\end{center}

\item Inversion. ($I_{pr}$) Consider $G$ (left) and ${\sf Inverted}(G)$ with $I_{pr}(G) = 0$, and $I_{pr}({\sf Inverted}(G)) = 1$. 
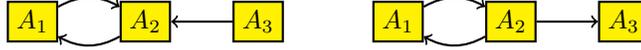
\begin{center}
\begin{tikzpicture}[->,thick]
\tikzstyle{every node}=[draw,rectangle,fill=yellow]
\node (A1)  at (0,0) {$A_1$};
\node (A2)  at (1.5,0) {$A_2$};
\node (A3) at (3,0) {$A_3$};
\path	(A2)[] edge[bend left] (A1);		
\path	(A1)[] edge[bend left] (A2);		
\path	(A3)[] edge[] (A2);		
\end{tikzpicture}
\hspace{1cm}
\begin{tikzpicture}[->,thick]
\tikzstyle{every node}=[draw,rectangle,fill=yellow]
\node (A1)  at (0,0) {$A_1$};
\node (A2)  at (1.5,0) {$A_2$};
\node (A3) at (3,0) {$A_3$};
\path	(A2)[] edge[bend left] (A1);		
\path	(A1)[] edge[bend left] (A2);		
\path	(A2)[] edge[] (A3);		
\end{tikzpicture}
\end{center}
($I_{ngr}$) Consider $G$ (above left) and ${\sf Inverted}(G)$ (above right) with $I_{ngr}(G) = 0$, and $I_{ngr}({\sf Inverted}(G))$ $= 3$. 
($I_{ust}$) Consider $G$ (left) and ${\sf Inverted}(G)$ (right) with $I_{ust}(G) = 0$, and $I_{ust}({\sf Inverted}(G)) = 2$. 
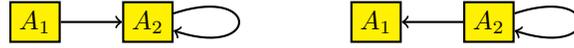
\begin{center}
\begin{tikzpicture}[->,thick]
\tikzstyle{every node}=[draw,rectangle,fill=yellow]
\node (A1)  at (0,0) {$A_1$};
\node (A2)  at (1.5,0) {$A_2$};
\path	(A1)[] edge[] (A2);		
\draw (A2) .. controls (3,0.5) and (3,-0.5) .. (A2);
\end{tikzpicture}
\hspace{1cm}
\begin{tikzpicture}[->,thick]
\tikzstyle{every node}=[draw,rectangle,fill=yellow]
\node (A1)  at (0,0) {$A_1$};
\node (A2)  at (1.5,0) {$A_2$};
\path	(A2)[] edge[] (A1);		
\draw (A2) .. controls (3,0.5) and (3,-0.5) .. (A2);
\end{tikzpicture}
\end{center}

\item Isomorphic invariance.  ($I_{pr}$,$I_{ngr}$,$I_{ust}$)  Follows directly from definition.

\item Disjoint additivity. ($I_{pr}$) Consider $G_1$ (left) and $G_2$ (right) 
where $I_{pr}(G_1) = 1$, $I_{pr}(G_2) = 1$, and $I_{pr}(G_1+G_2) = 3$. 
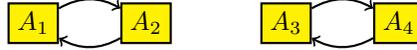
\begin{center}
\begin{tikzpicture}[->,thick]
\tikzstyle{every node}=[draw,rectangle,fill=yellow]
\node (A1)  at (0,0) {$A_1$};
\node (A2)  at (1.5,0) {$A_2$};
\path	(A2)[] edge[bend left] (A1);		
\path	(A1)[] edge[bend left] (A2);		
\end{tikzpicture}
\hspace{1cm}
\begin{tikzpicture}[->,thick]
\tikzstyle{every node}=[draw,rectangle,fill=yellow]
\node (A3)  at (0,0) {$A_3$};
\node (A4)  at (1.5,0) {$A_4$};
\path	(A4)[] edge[bend left] (A3);		
\path	(A3)[] edge[bend left] (A4);		
\end{tikzpicture}
\end{center}
%($I_{ngr}$) Consider $G_1$ (left) and $G_2$ (right) 
%where $I_{ngr}(G_1) = 0$, $I_{ngr}(G_2) = 0$, $I_{ngr}(G_1+G_2) = 2$. 
%\begin{center}
%\begin{tikzpicture}[->,thick]
%\tikzstyle{every node}=[draw,rectangle,fill=yellow]
%\node (A1)  at (0,0) {$A_1$};
%\node (A2)  at (1.5,0) {$A_2$};
%\path	(A2)[] edge[] (A1);		
%\end{tikzpicture}
%\hspace{1cm}
%\begin{tikzpicture}[->,thick]
%\tikzstyle{every node}=[draw,rectangle,fill=yellow]
%\node (A1)  at (0,0) {$A_1$};
%\node (A2)  at (1.5,0) {$A_2$};
%\path	(A1)[] edge[] (A2);		
%\end{tikzpicture}
%\end{center}
%($I_{ust}$) Consider $G_1$ (left) and $G_2$ (right) 
%where $I_{ust}(G_1) = 0$, $I_{ust}(G_2) = 0$, $I_{ust}(G_1+G_2) = 3$. 
%\begin{center}
%\begin{tikzpicture}[->,thick]
%\tikzstyle{every node}=[draw,rectangle,fill=yellow]
%\node (A1)  at (0,0) {$A_1$};
%\node (A2)  at (1.5,0) {$A_2$};
%\node (A3)  at (3,0) {$A_3$};
%\path	(A2)[] edge[] (A1);		
%\path	(A3)[] edge[] (A2);		
%\end{tikzpicture}
%\hspace{1cm}
%\begin{tikzpicture}[->,thick]
%\tikzstyle{every node}=[draw,rectangle,fill=yellow]
%\node (A1)  at (0,0) {$A_1$};
%\node (A3)  at (1.5,0) {$A_3$};
%\path	(A1)[] edge[] (A3);		
%\end{tikzpicture}
%\end{center}
($I_{ngr}$)  If $G_1$ and $G_2$ are disjoint, 
then ${\sf Extensions}_{\gr}(G_1)$ $\cap$ ${\sf Extensions}_{\gr}(G_2) = \emptyset$,
and  ${\sf Attackees}_{\gr}(G_1)$ $\cap$ ${\sf Attackees}_{\gr}(G_2) = \emptyset$, 
and  ${\sf Nodes}_{\gr}(G_1) \cap {\sf Nodes}_{\gr}(G_2) = \emptyset$.
Therefore, $I_{ngr}(G_1 + G_2) = I_{ngr}(G_1) + I_{ngr}(G_2)$.
($I_{ust}$)  If $G_1$ and $G_2$ are disjoint, 
${\sf min} \{ |X| \mid {\sf Extensions}_{\st}({\sf Induced}(G_1+G_2,X)) 
	\mbox{ s.t. } X \subseteq {\sf Nodes}(G_1+G_2) \}$
= ${\sf min} \{ |X| \mid {\sf Extensions}_{\st}({\sf Induced}(G_1,X)) 
	\mbox{ s.t. } X \subseteq {\sf Nodes}(G_1) \}$ 
	+ ${\sf min} \{ |X| \mid {\sf Extensions}_{\st}({\sf Induced}(G_2,X)) 
	\mbox{ s.t. } X \subseteq {\sf Nodes}(G_2) \}$.
Therefore, $I_{ust}(G_1 + G_2) = I_{ust}(G_1) + I_{ust}(G_2)$.

\item Super-additivity. ($I_{pr}$) Consider $G_1$ (left) and $G_2$ (right) 
where $I_{pr}(G_1) = 1$, $I_{pr}(G_2) = 0$, and $I_{pr}(G_1+G_2) = 0$. 
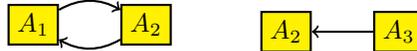
\begin{center}
\begin{tikzpicture}[->,thick]
\tikzstyle{every node}=[draw,rectangle,fill=yellow]
\node (A1)  at (0,0) {$A_1$};
\node (A2)  at (1.5,0) {$A_2$};
\path	(A2)[] edge[bend left] (A1);		
\path	(A1)[] edge[bend left] (A2);		
\end{tikzpicture}
\hspace{1cm}
\begin{tikzpicture}[->,thick]
\tikzstyle{every node}=[draw,rectangle,fill=yellow]
\node (A2)  at (0,0) {$A_2$};
\node (A3)  at (1.5,0) {$A_3$};
\path	(A3)[] edge[] (A2);		
\end{tikzpicture}
\end{center}
($I_{ngr}$) Consider $G_1$ (above left) and $G_2$ (above right) 
where $I_{ngr}(G_1) = 2$, $I_{ngr}(G_2) = 0$, and $I_{ngr}(G_1+G_2) = 0$. 
($I_{ust}$) Consider $G_1$ (left) and $G_2$ (right) 
where $I_{ust}(G_1) = 3$, $I_{ust}(G_2) = 0$, and $I_{ust}(G_1+G_2) = 0$. 
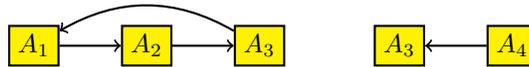
\begin{center}
\begin{tikzpicture}[->,thick]
\tikzstyle{every node}=[draw,rectangle,fill=yellow]
\node (A1)  at (0,0) {$A_1$};
\node (A2)  at (1.5,0) {$A_2$};
\node (A3)  at (3,0) {$A_3$};
\path	(A1)[] edge[] (A2);		
\path	(A2)[] edge[] (A3);		
\path	(A3)[] edge[bend right] (A1);		
\end{tikzpicture}
\hspace{1cm}
\begin{tikzpicture}[->,thick]
\tikzstyle{every node}=[draw,rectangle,fill=yellow]
\node (A3)  at (0,0) {$A_3$};
\node (A4)  at (1.5,0) {$A_4$};
\path	(A4)[] edge[] (A3);		
\end{tikzpicture}
\end{center}

\end{itemize}
\end{proof}

\begin{proposition}
The $I_{pr}$, $I_{ngr}$, $I_{ust}$ measures are pairwise incompatible.
Each is also pairwise incompatible with each of the 
$I_{dr}$, $I_{in}$, $I_{win}$, $I_{wou}$, $I_{cc}$, $I_{wcc}$ 
and $I_{ic}$ measures.
\end{proposition}

\begin{proof}
From the differences in satisfaction of properties in Proposition \ref{prop:properties:structure}, and Proposition \ref{prop:properties:extension}, the extension-based measures are pairwise incompatible with the structure-based measures. 
We now consider discriminating between the extension-based measures.
We can discriminate $I_{pr}$ from $I_{ngr}$ with graph $G_1$ (left) and $G_2$ (right),
where $I_{pr}(G_1) = 1$, $I_{pr}(G_2) = 7$, $I_{ngr}(G_1) = 2$, and $I_{ngr}(G_2) = 6$.
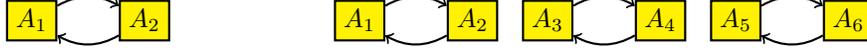
\begin{center}
\begin{tikzpicture}[->,thick]
\tikzstyle{every node}=[draw,rectangle,fill=yellow]
\node (A1)  at (0,0) {$A_1$};
\node (A2)  at (1.5,0) {$A_2$};
\path	(A1)[] edge[bend left] (A2);		
\path	(A2)[] edge[bend left] (A1);		
\end{tikzpicture}
\hspace{2cm}
\begin{tikzpicture}[->,thick]
\tikzstyle{every node}=[draw,rectangle,fill=yellow]
\node (A1)  at (0,0) {$A_1$};
\node (A2)  at (1.5,0) {$A_2$};
\node (A3)  at (2.5,0) {$A_3$};
\node (A4)  at (4,0) {$A_4$};
\node (A5)  at (5,0) {$A_5$};
\node (A6)  at (6.5,0) {$A_6$};
\path	(A2)[] edge[bend left] (A1);		
\path	(A1)[] edge[bend left] (A2);		
\path	(A4)[] edge[bend left] (A3);		
\path	(A3)[] edge[bend left] (A4);		
\path	(A5)[] edge[bend left] (A6);		
\path	(A6)[] edge[bend left] (A5);		
\end{tikzpicture}
\end{center}
We can discriminate $I_{pr}$ from $I_{ust}$ with graph $G_1$ (left) and $G_2$ (right),
where $I_{pr}(G_1) = 1$, $I_{pr}(G_2) = 3$, $I_{ust}(G_1) = 3$, and $I_{ust}(G_2) = 0$.
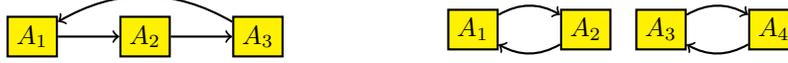
\begin{center}
\begin{tikzpicture}[->,thick]
\tikzstyle{every node}=[draw,rectangle,fill=yellow]
\node (A1)  at (0,0) {$A_1$};
\node (A2)  at (1.5,0) {$A_2$};
\node (A3)  at (3,0) {$A_3$};
\path	(A1)[] edge[] (A2);		
\path	(A2)[] edge[] (A3);		
\path	(A3)[] edge[bend right] (A1);		
\end{tikzpicture}
\hspace{2cm}
\begin{tikzpicture}[->,thick]
\tikzstyle{every node}=[draw,rectangle,fill=yellow]
\node (A1)  at (0,0) {$A_1$};
\node (A2)  at (1.5,0) {$A_2$};
\node (A3)  at (2.5,0) {$A_3$};
\node (A4)  at (4,0) {$A_4$};
%\node (A5)  at (5,0) {$A_5$};
%\node (A6)  at (6.5,0) {$A_6$};
\path	(A2)[] edge[bend left] (A1);		
\path	(A1)[] edge[bend left] (A2);		
\path	(A4)[] edge[bend left] (A3);		
\path	(A3)[] edge[bend left] (A4);		
%\path	(A5)[] edge[bend left] (A6);		
%\path	(A6)[] edge[bend left] (A5);		
\end{tikzpicture}
\end{center}
We can discriminate $I_{pr}$ from $I_{ust}$ with graph $G_1$ (left) and $G_2$ (right),
where $I_{ngr}(G_1) = 2$, $I_{ngr}(G_2) = 5$, $I_{ust}(G_1) = 0$, and $I_{ust}(G_2) = 3$.
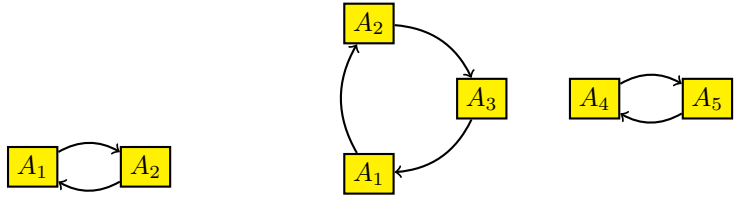
\begin{center}
\begin{tikzpicture}[->,thick]
\tikzstyle{every node}=[draw,rectangle,fill=yellow]
\node (A1)  at (0,0) {$A_1$};
\node (A2)  at (1.5,0) {$A_2$};
\path	(A2)[] edge[bend left] (A1);		
\path	(A1)[] edge[bend left] (A2);		
\end{tikzpicture}
\hspace{2cm}
\begin{tikzpicture}[->,thick]
\tikzstyle{every node}=[draw,rectangle,fill=yellow]
\node (A1)  at (0,0) {$A_1$};
\node (A2)  at (0,2) {$A_2$};
\node (A3) at (1.5,1) {$A_3$};
\node (A4) at (3,1) {$A_4$};
\node (A5) at (4.5,1) {$A_5$};
%\node (A6) at (6,1) {$A_6$};
%\node (A7) at (7.5,1) {$A_7$};
%\node (A8) at (9,1) {$A_8$};
\path	(A1)[] edge[bend left] (A2);		
\path	(A2)[] edge[bend left] (A3);		
\path	(A3)[] edge[bend left] (A1);	
%\path	(A3)[] edge[bend left] (A4);		
%\path	(A4)[] edge[bend left] (A3);		
\path	(A4)[] edge[bend left] (A5);		
\path	(A5)[] edge[bend left] (A4);	
%\path	(A6)[] edge[] (A7);	
%\path	(A7)[] edge[] (A8);			
\end{tikzpicture}
\end{center}
\end{proof}

In this subsection, we have considered an alternative class of measures based on the extensions of the graph. As can be seen from the properties that hold for them, they are very different in behaviour to the structured-based measures. On the one hand, measuring inconsistency in terms of the extensions seems a natural option, but on the other hand, they fail to satisfy most of the optional properties that we have considered, though this may be more a reflection of the nature of these properties, and that we need to consider a wider range of properties to justifiy and characterise extension-based measures.

%%%%%%%%%%%%%%%%%%%%%%%%%%%%%%%%%%%%%%%%%%%%%%%%%%
%%%%%%%%%%%%%%%%%%%%%%%%%%%%%%%%%%%%%%%%%%%%%%%%%%
%%%%%%%%%%%%%%%%%%%%%%%%%%%%%%%%%%%%%%%%%%%%%%%%%%
\section{Logic-based instantiations}
%%%%%%%%%%%%%%%%%%%%%%%%%%%%%%%%%%%%
\label{section:logic}

We now consider structured argumentation. This is where the arguments in the argument graph are instantiated with logical structure. We will proceed by reviewing a specific approach to structured argumentation called deductive argumentation, and then we consider two approaches to measuring inconsistency in deductive argumentation. The first is the degree of undercut approach proposed in \cite{BH05aaai,BH08book} and the second is the application of existing logic-based measures of inconsistency to deductive arguments in an argument graph.

%%%%%%%%%%%%%%%%%%%%%%%%%%%%%%%%%%%%%%%%%%%%%%%%%%
\subsection{Review of deductive argumentation}
%%%%%%%%%%%%%%%%%%%%%%%%%%%%%%%%%%%%%%%%%%%%%%

%In deductive reasoning, we start with some premises, and we derive a conclusion using one or more inference steps. Each inference step is infallible in the sense that it does not introduce uncertainty. In other words, if we accept the premises are valid, then we should accept that the intermediate conclusion of each inference step is valid, and therefore we should accept that the conclusion is valid. For example, if we accept that Philippe and Tony are having tea together in London is valid, then we should accept that Philippe is not in Toulouse (assuming the background knowledge that London and Toulouse are different places, and that nobody can be in different places at the same time). As another example, if we accept that Philippe and Tony are having an ice cream together in Toulouse is valid, then we should accept that Tony is not in London.Note, however, we do not need to believe or know that the premises are valid to apply deductive reasoning. Rather, deductive reasoning allows us to obtain conclusions that we can accept contingent on the validity of their premises. So for the first example above, the reader might not know whether or not Philippe and Tony are having tea together in London. However, the reader can accept that Philippe is not in Toulouse, contingent on the validity of these premises. Important alternatives to deductive reasoning in argumentation, include inductive reasoning, abductive reasoning, and analogical reasoning.  

In deductive argumentation, each argument is a pair where the first item is a set of premises that logically entails the second item according to some logical consequence relation (for a review see \cite{BesnardHunter2014tutorial}). So we have a logical language to express the set of premises, and the claim, and we have a logical consequence relation to relate the premises to the claim. So in order to construct argument graphs with deductive arguments, we need to specify the choice of logic (which we call the base logic) that we use to define arguments and counterarguments.

Classical logic is appealing as the choice of base logic as it reflects the richer deductive reasoning often seen in arguments arising in discussions and debates. 
We assume the usual propositional and predicate (first-order) languages for classical logic, and the usual 
the {\bf classical consequence relation}, denoted $\vdash$.
A {\bf classical knowledgebase} is a set of classical propositional or predicate formulae.

\begin{definition}
For a classical knowledgebase $\Phi$,
and a classical formula $\alpha$,
$\langle \Phi, \alpha \rangle$ is a {\bf classical argument}
iff $\Phi \vdash \alpha$ 
and $\Phi \not\vdash\bot$
and there is no proper subset $\Phi'$ of $\Phi$ 
such that $\Phi' \vdash \alpha$. 
For an argument $A = \langle \Phi, \alpha \rangle$, 
the function ${\sf Support}(A)$ returns $\Phi$ 
and the function ${\sf Claim}(A)$ returns $\alpha$. 
\end{definition}

%%%%%So a classical argument satisfies both minimality and consistency. We impose the consistency constraint because we want to avoid the useless inferences that come with inconsistency in classical logic (such as via ex falso quodlibet). 

\begin{example}
The following classical argument uses a universally quantified formula in contrapositive reasoning to obtain the claim about number 77.
\[
\tt \langle \{ \forall X. multipleOfTen(X) \rightarrow even(X), \neg even(77) \}, \neg multipleOfTen(77)  \rangle
\]
\end{example}

A counterargument is an argument that attacks another argument. In deductive argumentation, we define the notion of counterargument in terms of logical contradiction between the claim of the counterargument and the premises of claim of the attacked argument. 
Given the expressivity of classical logic (in terms of language and inferences), there are a number of natural ways to define counterarguments.

%%%%We explore some of the kinds of counterargument that can be specified for simple logic, classical logic and classical logic. 

\begin{definition}
\label{def:classical:attack}
Let $A$ and $B$ be two arguments. 
We define the following types of {\bf attack}.
\begin{itemize}
\item $A$ is a {\bf defeater} of $B$ 
if ${\sf Claim}(A) \vdash\neg \bigwedge {\sf Support}(B)$.
\item $A$ is a {\bf direct defeater} of $B$ if 
$\exists \phi \in {\sf Support}(B)$  s.t. ${\sf Claim}(A) \vdash \neg \phi$.
\item $A$ is a {\bf  undercut} of $B$ if 
$\exists \Psi \subseteq {\sf Support}(B)$ 
s.t. ${\sf Claim}(A) \equiv \neg \bigwedge \Psi$. 
\item $A$ is a {\bf  direct undercut} of $B$ if 
$\exists \phi \in {\sf Support}(B)$ s.t. ${\sf Claim}(A) \equiv \neg \phi$.
\item $A$ is a {\bf canonical undercut} of $B$ if 
${\sf Claim}(A) \equiv \neg \bigwedge {\sf Support}(B)$.
\item $A$ is a {\bf rebuttal} of $B$ if
${\sf Claim}(A) \equiv \neg {\sf Claim}(B)$.
\item $A$ is a {\bf  defeating rebuttal} of $B$ if 
${\sf Claim}(A) \vdash \neg {\sf Claim}(B)$.
\end{itemize} 
%\begin{center}
%\begin{tabular}{ll}
%\hline
%Name for attack by $A$ on $B$ & Condition\\
%\hline
%classical defeater & ${\sf Claim}(A) \vdash\neg \bigwedge {\sf Support}(B)$\\
%classical direct defeater & there is $\phi \in {\sf Support}(B)$  s.t. ${\sf Claim}(A) \vdash \neg \phi$\\
%classical undercut & $\Psi \subseteq {\sf Support}(B)$ s.t.${\sf Claim}(A) \equiv \neg \bigwedge \Psi$\\
%classical direct undercut & $\phi \in {\sf Support}(B)$ s.t. ${\sf Claim}(A) \equiv \neg \phi$\\
%classical canonical undercut & ${\sf Claim}(A) \equiv \neg \bigwedge {\sf Support}(B)$\\
%classical rebuttal & ${\sf Claim}(A) \equiv \neg {\sf Claim}(B)$\\
%classical defeating rebuttal & ${\sf Claim}(A) \vdash \neg {\sf Claim}(B)$\\
%\hline
%\end{tabular}
%\end{center}
%Note, in the rest of this section, we will drop the term "classical" when we discuss these types of attack. 
\end{definition}

To illustrate these different notions of counterargument, we consider the following examples,
and we relate these definitions in Figure \ref{fig:classical:attack} where we show that classical defeaters are the most general of these definitions.

\begin{example}
Let $\Delta = \{ 
	a\vee b, a \leftrightarrow b,
	c\rightarrow a,
	\neg a\wedge \neg b,
	a,b,c,
	a \rightarrow b,
	\neg a, \neg b, \neg c
	\}$
\[
\begin{array}{l}
\langle \{ a \vee b, c \}, (a \vee b) \wedge c \rangle
	\mbox{ is a defeater of } 
	\langle \{ \neg a, \neg b \}, \neg a \wedge \neg b \rangle\\
	
\langle \{ a \vee b, c \}, (a \vee b) \wedge c \rangle
	\mbox{ is a direct defeater of }
	\langle \{ \neg a \wedge \neg b \}, \neg a \wedge \neg b \rangle\\

\langle \{ \neg a\wedge \neg b \}, \neg (a \wedge b) \rangle
	\mbox{ is a undercut of } 
	\langle \{ a,b,c \}, a \wedge b \wedge c \rangle\\
	
\langle \{ \neg a\wedge \neg b \}, \neg a \rangle
	\mbox{ is a direct undercut of } 
	\langle \{ a,b,c \}, a \wedge b \wedge c \rangle\\

\langle \{ \neg a\wedge \neg b \}, \neg (a \wedge b \wedge c) \rangle
	\mbox{ is a canonical undercut of } 
	\langle \{ a,b,c \}, a \wedge b \wedge c \rangle\\

\langle \{ a, a \rightarrow b \}, b \vee c  \rangle
	\mbox{ is a rebuttal of } 
	\langle \{ \neg a \wedge \neg b, \neg c \}, \neg(b \vee c) \rangle\\
	
\langle \{ a, a \rightarrow b \}, b  \rangle
	\mbox{ is a defeating rebuttal of } 
	\langle \{ \neg a \wedge \neg b, \neg c \}, \neg(b \vee c) \rangle\\
\end{array}
\]

\end{example}

\begin{figure}
\begin{center}
\begin{tikzpicture}[->,>=latex,thick]

\node (d) [text centered,text width = 3cm,shape=rectangle,fill=blue!20,draw] {defeater};
\node (dd) [below left=of d,text centered,text width = 3cm,shape=rectangle,fill=blue!20,draw] {direct defeater};
\node (u) [below=of d,text centered,text width = 3cm,shape=rectangle,fill=blue!20,draw] {undercut};
\node (dr) [below right=of d,text centered,text width = 3cm,shape=rectangle,fill=blue!20,draw] {defeating rebuttal};
			
\node (du) [below=of dd,text centered,text width = 3cm,shape=rectangle,fill=blue!20,draw] {direct undercut};
\node (cu) [below=of u,text centered,text width = 3cm,shape=rectangle,fill=blue!20,draw] {canonical undercut};
\node (r) [below=of dr,text centered,text width = 3cm,shape=rectangle,fill=blue!20,draw] {rebuttal};
			
			\path	(dd) edge  (d);
			\path	(u) edge (d);
			\path	(dr) edge (d);
			
			\path	(du) edge (dd);
			\path	(du) edge (u);
			\path	(cu) edge (u);
			\path	(r) edge (dr);
			
\end{tikzpicture} 
\end{center}
\caption{\label{fig:classical:attack}We can represent the containment between the attack relations as above where an arrow from $R_1$ to $R_2$
indicates that $R_1 \subseteq R_2$. }
%%%%For proofs, see \cite{BH01} and \cite{GH11}.}
\end{figure}
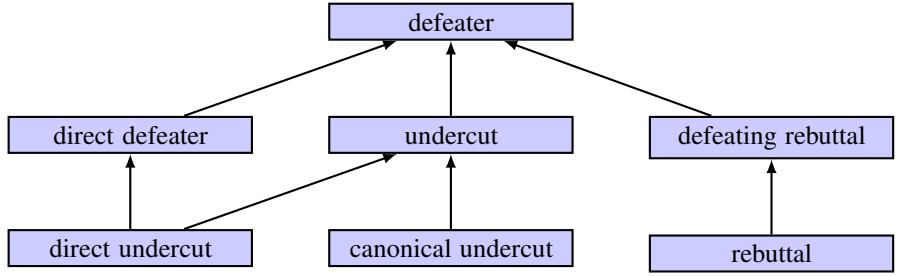

\begin{figure}
\begin{center}
\begin{tikzpicture}[->,>=latex,thick]
\node (a1) [text centered,text width=10cm,shape=rectangle,fill=yellow,draw] {
						\[
						\begin{array}{l}
						\mbox{\small \tt bp(high)}\\
						\mbox{\small \tt ok(diuretic)}\\
						\mbox{\small \tt bp(high)} \wedge \mbox{\small \tt ok(diuretic)} \rightarrow \mbox{\small \tt give(diuretic)}\\
						\neg \mbox{\small \tt ok(diuretic)} \vee \neg \mbox{\small \tt ok(betablocker)}\\
						\hline
						\mbox{\small \tt give(diuretic)} \wedge \neg \mbox{\small \tt ok(betablocker)}						
						\end{array}
						\]};
			\node (a2) [below=of a1,text centered,text width=10cm,shape=rectangle,fill=yellow,draw] {
						\[
						\begin{array}{l}
						\mbox{\small \tt bp(high)} \\
						\mbox{\small \tt ok(betablocker)}\\
						\mbox{\small \tt bp(high)} \wedge \mbox{\small \tt ok(betablocker)} \rightarrow \mbox{\small \tt give(betablocker)}\\
						\neg \mbox{\small \tt ok(diuretic)} \vee \neg \mbox{\small \tt ok(betablocker)}\\
						\hline
						\mbox{\small \tt give(betablocker)} \wedge \neg \mbox{\small \tt ok(diuretic)}
						\end{array}
						\]};
			\node (a3) [below=of a2,text centered,text width=8cm,shape=rectangle,fill=yellow,draw] {
						\[
						\begin{array}{l}
						\mbox{\small \tt symptom(emphysema)}, \\
						\mbox{\small \tt symptom(emphysema)} \rightarrow \neg \mbox{\small \tt ok(betablocker)} \\
						\hline
						\neg \mbox{\small \tt ok(betablocker)}
						\end{array}
						\]};
			\path	(a1)[bend left] edge node[auto] {} (a2);
			\path	(a2)[bend left] edge node[] {} (a1);
			\path	(a3) edge node[] {} (a2);
\end{tikzpicture}
\end{center}
\caption{\label{fig:classical:medical}An instantiated argument graph  for the abstract argument graph in Example \ref{ex:intro1}. The atom ${\tt bp(high)}$ denotes that the patient has high blood pressure.  The top two arguments rebut each other (i.e. the attack is  defeating rebut). For this, each argument has an integrity constraint in the premises that says that it is not ok to give both betablocker and diuretic. So the top argument is attacked on the premise ${\tt ok(diuretic)}$ and the middle argument is attacked on the premise ${\tt ok(betablocker)}$. So we are using the ${\tt ok}$ predicate as a normality condition for the rule to be applied.}
\end{figure}
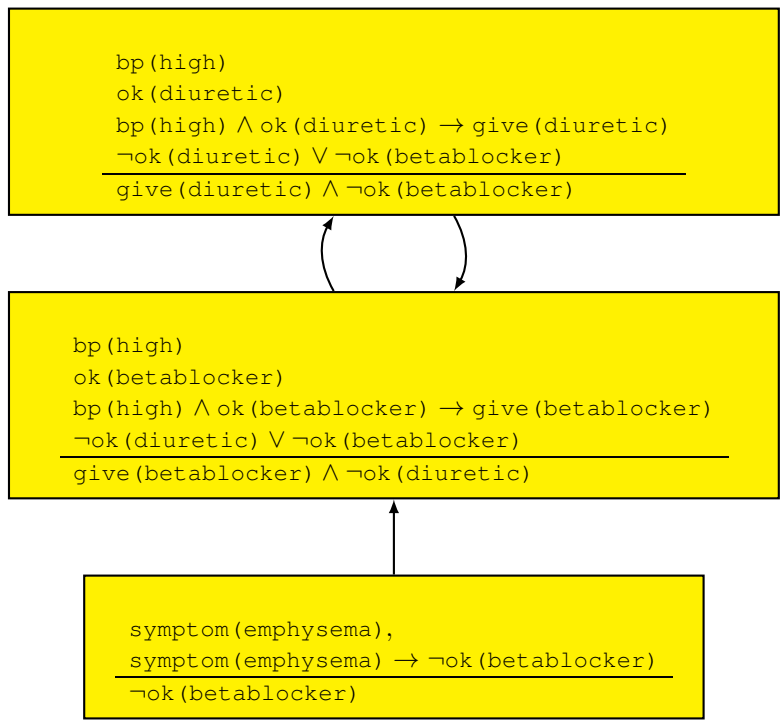

%%%%%%Using simple logic, the definitions for counterarguments against the support of another argument are limited to attacking just one of the items in the support. In contrast,  using classical logic, a counterargument can be against more than one item in the support. For example, in Example \ref{ex:classical:counterarguments:airline}, the undercut is not attacking an individual premise but rather saying that two of the premises are incompatible (in this case that the premises ${\tt lowCostFly}$ and ${\tt luxuryFly}$ are incompatible).  

\begin{example}
\label{ex:classical:counterarguments:airline}
Consider the following arguments. 
$A_1$ is attacked by $A_2$ as $A_2$ is an undercut of $A_1$ though it is neither a direct undercut nor a canonical undercut.
Essentially, the attack says that the flight cannot be both a low cost flight and a luxury flight.
\[
\begin{array}{l}
A_1 = \langle \{ {\tt lowCostFly}, {\tt luxFly}, {\tt lowCostFly} \wedge {\tt luxFly} \rightarrow {\tt goodFly} \}, {\tt goodFly} \rangle\\						
A_2 = \langle \{ \neg {\tt lowCostFly} \vee \neg {\tt luxFly} \}, \neg {\tt lowCostFly} \vee \neg {\tt luxFly} \rangle\\
\end{array}
\]
\end{example}

Trivially, undercuts are defeaters but it is also quite simple to establish that rebuttals are defeaters. Furthermore,  if an argument has defeaters then it has undercuts. It may happen that an argument has defeaters but no rebuttals as illustrated next.

\begin{example}
Let $\Delta=\{\tt \neg containsGarlic \wedge goodDish,\neg goodDish\}$.
Then the following argument has at least one defeater 
but no rebuttal.
\[
\langle\{\tt \neg containsGarlic \wedge goodDish \}, \neg containsGarlic \rangle
\]
\end{example}

There are some important differences between rebuttals and undercuts that can be seen in the following examples. 

\begin{example}
Consider the following arguments. The first argument $A_1$ is a direct undercut to the second argument $A_2$, but neither rebuts each other.
Furthermore, $A_1$ ``agrees" with the claim of $A_2$ since the premises of $A_1$ could be used for an alternative argument with the same claim as $A_2$.
\[
\begin{array}{l}
A_1 = {\langle\{\tt \neg containsGarlic \wedge \neg goodDish \}, \neg containsGarlic \rangle} \\
A_2 = {\langle\{\tt containsGarlic, containsGarlic \rightarrow  \neg goodDish \}, \neg goodDish \rangle} \\
\end{array}
\]
\end{example}

\begin{example}
Consider the following arguments. The first argument is a rebuttal of the second argument, but it is not an undercut because the claim of the first argument is not equivalent to the negation of some subset of the premises of the second argument.
\[
\begin{array}{l}
A_1 = {\langle\{\tt goodDish \}, goodDish \rangle} \\
A_2 = {\langle\{\tt containsGarlic, containsGarlic \rightarrow  \neg goodDish \}, \neg goodDish \rangle} \\
\end{array}
\]
\end{example}

So an undercut for an argument need not be a rebuttal for that argument, and a rebuttal for an argument need not be an undercut for that argument.

An instantiated argument graph is an argument graph where each node is a deductive argument, and each arc is an attack confirming to the definitions for attack (Definition \ref{def:classical:attack}). We provide illustrations of instantiated argument graphs in the following example and in Figure \ref{fig:classical:medical}.

\begin{example}
\label{ex:classical:airline}
Consider the following argument graph
where $A_1$ is ``The flight is low cost and luxury, therefore it is a good flight", 
and $A_2$ is ``A flight cannot be both low cost and luxury".
 \begin{center}
\begin{tikzpicture}[->,>=latex,thick,every node/.style={draw,rectangle,fill=yellow}]
			\node (1) [] {$A_1$};
			\node (2) [right=of 1] {$A_2$};
			\path	(2)[] edge (1);
\end{tikzpicture}
\end{center}
For this, we instantiate the arguments in the above abstract argument graph to give the following instantiated argument graph where $A_2$ is an undercut to $A_1$. 
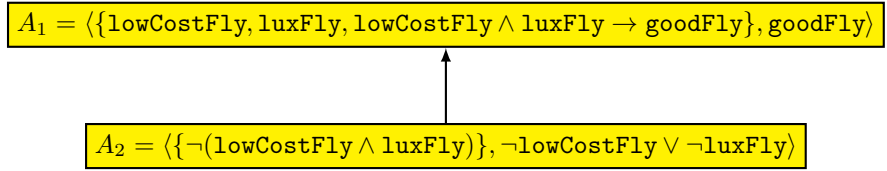
\begin{center}
\begin{tikzpicture}[->,>=latex,thick,every node/.style={draw,rectangle,fill=yellow}]
			\node (1) [] {$A_1 = \langle \{ {\tt lowCostFly}, {\tt luxFly}, {\tt lowCostFly} \wedge {\tt luxFly} \rightarrow {\tt goodFly} \}, {\tt goodFly} \rangle$};
			\node (2) [below=of 1] {$A_2 = \langle \{ \neg ( {\tt  lowCostFly} \wedge {\tt luxFly}) \}, \neg {\tt lowCostFly} \vee \neg {\tt luxFly} \rangle$};
			\path	(2)[] edge (1);
\end{tikzpicture}
\end{center}
\end{example}

Perhaps the first paper to consider instantiating argument graphs with deductive arguments based on classical logic is by Cayrol  \cite{Cayrol95} using direct undercut. For more details on deductive argumentation and how it can be used to instantiate argument graphs, see \cite{BesnardHunter2014tutorial}.

%%%%%%%%%%%%%%%%%%%%%%%%%%%%%%%%%%%%%%%%%%%%%%%%%%
\subsection{Degree of undercut}
%%%%%%%%%%%%%%%%%%%%%%%%%%%%%%%%
\label{section:more}

An argument conflicts with each of its undercuts, by the very definition  of an undercut. Now, some may conflict more than others, and some may conflict a little while  others conflict a lot.

\begin{example}
Consider the following argument graph $G$.
%%% $G_1$ (left), $G_2$ (middle), and $G_3$ (right). 
Each undercut has a premise that negates some or all of the premises in the root. 
The left child has the weakest premise which can be read as saying that one of the premises in the root is false without saying which, the middle child says that one of the premises in the root is false and states which one, and the right child says that all of the premises are false. 
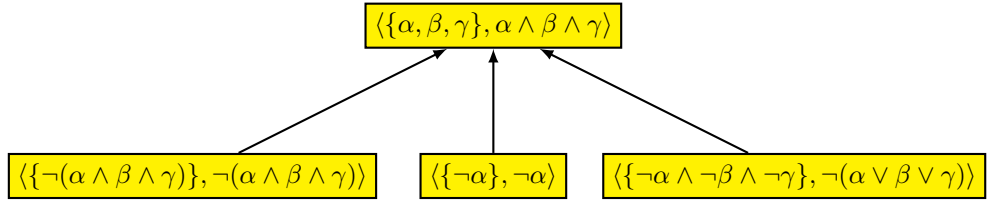
\begin{center}
\begin{tikzpicture}[->,>=latex,thick,every node/.style={draw,rectangle,fill=yellow}]
			\node (1) [] at (4,2) {$\langle \{ \alpha, \beta, \gamma \}, \alpha \wedge \beta \wedge \gamma \rangle$};
			\node (2) [] at (0,0) {$\langle \{ \neg (\alpha \wedge \beta \wedge \gamma) \}, \neg(\alpha\wedge\beta\wedge\gamma) \rangle$};
			\node (3) [] at (4,0) {$\langle \{ \neg \alpha \}, \neg\alpha \rangle$};
			\node (4) [] at (8,0) {$\langle \{ \neg \alpha \wedge \neg \beta \wedge \neg \gamma \}, \neg(\alpha\vee\beta\vee\gamma) \rangle$};
			\path	(2)[] edge (1);
			\path	(3)[] edge (1);
			\path	(4)[] edge (1);
\end{tikzpicture}
\end{center}
\end{example}

%\begin{tikzpicture}[->,>=latex,thick,every node/.style={draw,rectangle,fill=yellow}]
%			\node (1) [] {$\langle \{ \alpha \wedge \beta \wedge \gamma \}, \alpha \wedge \beta \wedge \gamma \rangle$};
%			\node (2) [below=of 1] {$\langle \{ \neg (\alpha \wedge \beta \wedge \gamma) \}, \neg(\alpha\wedge\beta\wedge\gamma) \rangle$};
%			\path	(2)[] edge (1);
%\end{tikzpicture}
%\begin{tikzpicture}[->,>=latex,thick,every node/.style={draw,rectangle,fill=yellow}]
%			\node (1) [] {$\langle \{ \alpha \wedge \beta \wedge \gamma \}, \alpha \wedge \beta \wedge \gamma \rangle$};
%			\node (2) [below=of 1] {$\langle \{ \neg \alpha \}, \neg\alpha \rangle$};
%			\path	(2)[] edge (1);
%\end{tikzpicture}
%\begin{tikzpicture}[->,>=latex,thick,every node/.style={draw,rectangle,fill=yellow}]
%			\node (1) [] {$\langle \{ \alpha \wedge \beta \wedge \gamma \}, \alpha \wedge \beta \wedge \gamma \rangle$};
%			\node (2) [below=of 1] {$\langle \{ \neg \alpha \wedge \neg \beta \wedge \neg \gamma \}, \neg(\alpha\wedge\beta\wedge\gamma) \rangle$};
%			\path	(2)[] edge (1);
%\end{tikzpicture}

By these simple examples of undercuts, we see that there can be a  difference in the amount of conflict between supports, and hence can be taken as an intuitive starting point for defining the degree of undercut that an argument has against its parent. To address this, the degree of undercut is a measure of the conflict between a pair of arguments based on the supports of these arguments \cite{BH05aaai,BH08book}. There are some alternatives for defining the degree of undercut,  and we review one of these proposals in the next subsection.

\subsubsection{Degree of undercut based on Dalal distance}
%%%%%%%%%%%%%%%%%%%%%%%%%%%%%%%%%%%%%%%%%%%%%%%%%%%%%%%
\index{Degree of undercut!Dalal distance}

%%%The first idea that comes to mind when considering specifying a degree of undercut is to use a distance.

In this section, we investigate a degree of undercut based on the distance between pairs of models. For this, we use the Dalal distance \cite{Dalal1988}.

\begin{definition}
For the language ${\cal L}$, let ${\sf Atoms}({\cal L})$ be the set of atoms used in the language (and so the formulae in ${\cal L}$ are composed from ${\sf Atoms}({\cal L})$ and the logical connectives using the usual inductive definition). 
%%%%%For a set of propositional formulae $\Phi$, let ${\sf Atoms}(\Phi)$ be the set of atoms used in the formulae in $\Phi$.  
\end{definition}

\begin{definition}
Let $\Pi$ be a finite non-empty subset of ${\sf Atoms}({\cal L})$ and let $w_i,w_j \in \wp(\Pi)$.
The {\bf Dalal distance} between $w_i$ and $w_j$, denoted ${\sf Dalal}(w_i,w_j)$, 
is the difference in the number of atoms assigned true:
\[
{\sf Dalal}(w_i,w_j) = |w_i - w_j| + |w_j - w_i|
\]
\end{definition}

\begin{example}
Let $w_1 = \{\alpha, \gamma, \delta\}$ and $w_2 = \{ \beta, \gamma\}$ where $\{\alpha, \beta, \gamma, \delta\} \subseteq \Pi$.
Then, 
\[
{\sf Dalal}(w_1,w_2)  =  |\{\alpha, \delta\}| + |\{\beta\}|  =  3
\]
\end{example}

To evaluate the conflict between the support of an argument 
$A$ and the support of an undercut $A'$, we consider the models of ${\sf Support}(A)$ and ${\sf Support}(A')$, restricted to  a set of atoms $\Pi$. For this, we require the following definition.

\begin{definition}
Let $\Pi$ is a finite non-empty subset of ${\sf Atoms}({\cal L})$, 
let $\Phi$ be a set of formulae,
and let $\models$ be the classical satisfaction relation.
\[
[\sf Models(\Phi,\Pi) = \{ w \in \wp(\Pi) \mid \forall \phi \in \Phi \mbox{ and } w \models \phi \}
\]
\end{definition}

 %%%%(i.e. ${\sf Models}(\Phi,\Pi)$ and ${\sf Models}(\Psi,\Pi)$).

\begin{example}
Let $\Phi = \{ \alpha \wedge \delta, \neg \phi, \gamma \vee \delta, \neg \psi, \beta \vee \gamma\} \subseteq \Delta$, and let $\Pi$ be $\{\alpha,\beta,\gamma,\delta,\phi\} \subseteq {\sf Atoms}({\cal L})$. Should it be the case that $\Pi \neq {\sf Atoms}({\cal L})$, the set of all the models 
of $\Phi$ is a proper superset of ${\sf Models}(\Phi,\Pi)$ as the latter consists exactly 
of the following models.
\[
\{\alpha,\beta,\gamma,\delta\}
\hspace{1cm}
\{\alpha, \beta, \delta\}
\hspace{1cm}
\{\alpha,\gamma,\delta\}
\]
\end{example}

To evaluate the conflict between two sets of formulae, we take a pair of models restricted to $\Pi$, one for each set, such that the Dalal distance is minimized. The degree of conflict is this distance divided by the maximum possible Dalal distance between a pair of models 
(i.e.\ $\log_2$ of the total number of models in $\wp(\Pi)$ which is $|\Pi|$). 
%This generalizes degree of conflict proposed in \cite{Hun04agb}.

\begin{definition}
\label{def:rational:conflict:dalal}
The {\bf degree of conflict} wrt $\Pi$, denoted ${\sf Conflict}(\Phi,\Psi,\Pi)$, is:
\[
%{\sf Conflict}(\Phi,\Psi,\Pi) 
\frac{\min\{ {\sf Dalal}(w_\Phi,w_\Psi) 
			\mid w_\Phi \in {\sf Models}(\Phi,\Pi), w_\Psi \in {\sf Models}(\Psi,\Pi) \}}{|\Pi|}
\]
\end{definition}

\begin{example}
Let $\Pi = \{\alpha,\beta,\gamma,\delta \}$.
\[
\begin{array}{l}
{\sf Conflict}(\{\alpha \wedge \beta \wedge \gamma \wedge \delta\}, \{\neg \alpha \vee \neg \beta \vee \neg \gamma\},\Pi) = 1/4\\
{\sf Conflict}(\{\alpha \wedge \beta \wedge \gamma \wedge \delta\}, \{\neg (\alpha \vee \beta)\}, \Pi) = 2/4\\
{\sf Conflict}(\{\alpha \wedge \beta \wedge \gamma \wedge \delta\}, \{\neg \alpha \wedge \neg \beta \wedge \neg \gamma\},\Pi) = 3/4
\end{array}
\]
\end{example}

We obtain a degree of undercut by applying ${\sf Conflict}$ to supports as defined next. 

\begin{definition}
Let $A_i = \langle \Phi_i, \alpha_i \rangle$  
and $A_j = \langle \Phi_j, \alpha_j \rangle$ be arguments.
\[{\sf Degree}(A_i,A_j)  = {\sf Conflict}(\Phi_i,\Phi_j,{\sf Atoms}(\Phi_i\cup\Phi_j))\]
\end{definition}

Clearly, if $A_i$ is an undercut for $A_j$, then ${\sf Degree}(A_i,A_j) > 0$.

\begin{example}
Below we give the values for the degree of undercut of the argument
$\langle \{\alpha, \beta, \gamma\}, \ldots \rangle$ (the consequent is unimportant here) by some of its canonical undercuts:
\[\begin{array}{l}
{\sf Degree}(\langle \{\alpha, \beta, \gamma\}, \ldots \rangle, \langle \{\neg \alpha \wedge \neg \beta \wedge \neg \gamma\}, \ldots \rangle) = 1\\
{\sf Degree}(\langle \{\alpha, \beta, \gamma\}, \ldots \rangle, \langle \{\neg \alpha \wedge \neg \beta\}, \ldots \rangle) = 2/3\\
{\sf Degree}(\langle \{\alpha, \beta, \gamma\}, \ldots \rangle, \langle \{\neg \alpha \vee \neg \beta \vee \neg \gamma\}, \ldots \rangle) = 1/3\\
{\sf Degree}(\langle \{\alpha, \beta, \gamma\}, \ldots \rangle, \langle \{\neg \alpha\}, \ldots \rangle) = 1/3
\end{array}\]
\end{example}

\begin{example}
\label{ex:degreeofundercut:4}
Consider the following argument graph where ${\sf Degree}( A_1, A_2) = 1/3$, ${\sf Degree}( A_1, A_3) = 2/3$, and ${\sf Degree}( A_1, A_4) = 3/4$. 
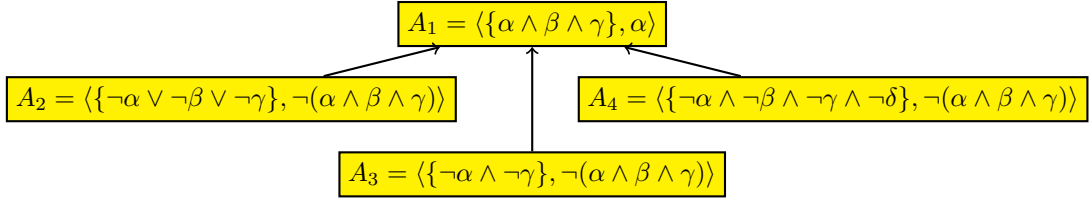
\begin{center}
\begin{tikzpicture}[->,thick]
\tikzstyle{every node}=[draw,rectangle,fill=yellow]
\node (A1)  at (4,2) {$A_1 =  \langle \{ \alpha \wedge \beta \wedge \gamma \}, \alpha \rangle$};
\node (A2)  at (0,1) {$A_2 = \langle \{ \neg \alpha \vee \neg \beta \vee \neg \gamma \}, \neg( \alpha \wedge \beta \wedge \gamma) \rangle$};
\node (A3)  at (4,0) {$A_3 = \langle \{ \neg \alpha \wedge \neg \gamma \}, \neg( \alpha \wedge \beta \wedge \gamma) \rangle$};
\node (A4) at (8,1) {$A_4 = \langle \{ \neg \alpha \wedge \neg \beta \wedge \neg \gamma \wedge \neg \delta\}, \neg( \alpha \wedge \beta \wedge \gamma) \rangle$};
\path	(A2)[] edge[] (A1);		
\path	(A3)[] edge[] (A1);		
\path	(A4)[] edge[] (A1);		
\end{tikzpicture}
\end{center}
%\[\begin{array}{lccc}
%& T_7 & T_8 & T_9 \\
%&\\
%A & \langle \{ \alpha \wedge \beta \wedge \gamma \}, \alpha \rangle 
%& \langle \{ \alpha \wedge \beta \wedge \gamma \}, \alpha\rangle 
%& \langle \{ \alpha \wedge \beta \wedge \gamma \}, \alpha \rangle \\
%& \uparrow & \uparrow & \uparrow \\
%A' & \langle \{ \neg \alpha \vee \neg \beta \vee \neg \gamma \}, \Diamond \rangle
%& \langle \{ \neg \alpha \wedge \neg \gamma \}, \Diamond \rangle
%& \langle \{ \neg \alpha \wedge \neg \beta \wedge \neg \gamma \wedge \neg \delta\}, \Diamond \rangle\\
%&\\
%&{\sf Degree}_\mathsf{C}( A, A') = 1/3 & {\sf Degree}_\mathsf{C}( A, A') = 2/3 & {\sf Degree}_\mathsf{C}( A, A') = 3/4
%\end{array}\]
\end{example}

\begin{example}
More generally, let 
$A_1 = \langle \{ \neg (\alpha_1 \vee\ldots\vee \alpha_n) \}, \neg (\alpha_1 \wedge\ldots\wedge \alpha_n) \rangle$,
$A_2 = \langle \{ \neg \alpha_1 \vee\ldots\vee \neg \alpha_n \}, \neg (\alpha_1 \wedge\ldots\wedge \alpha_n) \rangle$,
$A_3 = \langle \{ \neg \alpha_1 \}, \neg (\alpha_1 \wedge\ldots\wedge \alpha_n) \rangle$,
$A_4 = \langle \{ \alpha_1 \wedge\ldots\wedge \alpha_n \}, \alpha_1 \rangle$.
\[
\begin{array}{l}
{\sf Degree}( A_4, A_1) = n/n\\
{\sf Degree}( A_4, A_2) = 1/n\\
{\sf Degree}( A_4, A_3) = 1/n
\end{array}
\]
\end{example}

The above examples indicate how the ${\sf Degree}$ measure differentiates between different kinds of attack, and the following result shows the ${\sf Degree}$ measure has the basic properties that we require. 

\begin{proposition}
Let $A_i = \langle \Phi_i, \alpha_i \rangle$  
and $A_j = \langle \Phi_j, \alpha_j \rangle$ be arguments.
\[
\begin{array}{ll}
(1) & 0 \leq {\sf Degree}(A_i,A_j) \leq 1\\
(2) & {\sf Degree}(A_i,A_j) = {\sf Degree}_\mathsf{C}(A_j,A_i)\\
(3) & {\sf Degree}(A_i,A_j) = 0 \mbox{ iff } \Phi_i \cup \Phi_j \not\vdash \bot
\end{array}
\]
\end{proposition}

So the degree of undercut gives a value in the unit interval to represent the how much two arguments differ in terms of their premises. In addition to this proposal, \cite{BH08book} presents some alternative proposals for defining degree of undercut.  

%%%%%%%%%%%%%%%%%%%%%%%%%%%%%%%%%%%%%%%%%%%%%%%%%%%%%%%
\subsubsection{Cumulative degree of undercut}
%%%%%%%%%%%%%%%%%%%%%%%%%%%%%%%%%%%%%%%%%%%%%%%%%%%%%%%

We now consider how we can harness the notion of degree of undercut as  an inconsistency measure for an argument graph.

\begin{definition}
Let $G$ be an argument graph.
The {\bf cumulative degree of undercut} in $G$, denoted ${\sf I_{cu}}(G)$, 
is given by 
\[
I_{cu}(G) = \sum_{(A_i,A_j) \in {\sf Arcs}(G)} {\sf Degree}(A_i,A_j)
\]
\end{definition}

\begin{example}
Consider the argument graph in Example \ref{ex:degreeofundercut:4}. 
For this, $I_{cu}(G) = 1/3 + 2/3 + 3/4 = 7/4$. 
%\begin{center}
%\begin{tikzpicture}[->,thick]
%\tikzstyle{every node}=[draw,rectangle,fill=yellow]
%\node (A1)  at (4,2) {$A_1 =  \langle \{ \alpha \wedge \beta \wedge \gamma \}, \alpha \rangle$};
%\node (A2)  at (0,1) {$A_2 = \langle \{ \neg \alpha \vee \neg \beta \vee \neg \gamma \}, \neg( \alpha \wedge \beta \wedge \gamma) \rangle$};
%\node (A3)  at (4,0) {$A_3 = \langle \{ \neg \alpha \wedge \neg \gamma \}, \neg( \alpha \wedge \beta \wedge \gamma) \rangle$};
%\node (A4) at (8,1) {$A_4 = \langle \{ \neg \alpha \wedge \neg \beta \wedge \neg \gamma \wedge \neg \delta\}, \neg( \alpha \wedge \beta \wedge \gamma) \rangle$};
%\path	(A2)[] edge[] (A1);		
%\path	(A3)[] edge[] (A1);		
%\path	(A4)[] edge[] (A1);		
%\end{tikzpicture}
%\end{center}
\end{example}

\begin{proposition}
The $I_{cu}$ measure is a gra[h-based inconsistency measure according 
to Definition \ref{def:inconsistencymeasure}.
\end{proposition}

\begin{proof}
(Consistency) Assume ${\sf Arcs}(G) = \emptyset$. 
Therefore, $\sum_{(A_i,A_j) \in {\sf Arcs}(G)} {\sf Degree}(A_i,A_j) = 0$.
Therefore, $I_{cu}(G) = 0$. 
(Freeness) 
Assume ${\sf Nodes}(G) = {\sf Nodes}(G')\setminus\{A\}$ 
and ${\sf Arcs}(G) = {\sf Arcs}(G')$. 
Therefore, $\sum_{(A_i,A_j) \in {\sf Arcs}(G)} {\sf Degree}(A_i,A_j)$
= $\sum_{(A_i,A_j) \in {\sf Arcs}(G')} {\sf Degree}(A_i,A_j)$.
Therefore, $I_{cu}(G) = I_{cu}(G')$. 
\end{proof}

%%%%%%\todo{For following result - we assume e.g. ${\sf Arcs}(G) \subseteq {\sf Arcs}(G')$ means the content of the nodes is the same - this needs being made explicit. }

\begin{proposition}
\label{prop:properties:cu}
The $I_{cu}$ measure satisfies Monotonicity,
Inversion,
Isomorphic invariance,
Disjoint additivity,
and 
Super-additivity.
\end{proposition}

\begin{proof}
(Monotonicity) 
Assume $G \sqsubseteq G'$. 
So ${\sf Arcs}(G) \subseteq {\sf Arcs}(G')$.
So, $\sum_{(A_i,A_j) \in {\sf Arcs}(G)} {\sf Degree}(A_i,A_j)$
$\leq$ $\sum_{(A_i,A_j) \in {\sf Arcs}(G')} {\sf Degree}(A_i,A_j)$.
So, $I_{cu}(G) \leq I_{cu}(G')$. 
(Inversion) Assume $G' = {\sf Invert}(G)$.
Since ${\sf Degree}$ is symmetric, 
\[
\sum_{(A_i,A_j) \in {\sf Arcs}(G)} {\sf Degree}(A_i,A_j)
= \sum_{(A_i,A_j) \in {\sf Arcs}({\sf Invert}(G))} {\sf Degree}(A_j,A_i)
\]
Therefore, $I_{cu}(G) = I_{cu}({\sf Invert}(G))$. 
(Isomorphic invariance) Similar to proof for inversion.
(Disjoint additivity) Assume $G_1$ and $G_2$ are disjoint.
Therefore, $\sum_{(A_i,A_j) \in {\sf Arcs}(G)} {\sf Degree}(A_i,A_j)$ =
\[
\sum_{(A_i,A_j) \in {\sf Arcs}(G_1)} {\sf Degree}(A_i,A_j)
+ \sum_{(A_i,A_j) \in {\sf Arcs}((G_2)} {\sf Degree}(A_j,A_i)
\]
Therefore, $I_{cu}(G) = I_{cu}(G_1) + I_{cu}(G_2)$.
(Super-additivity) Similar to proof for disjoint additivity. 
\end{proof}

\begin{proposition}
The $I_{cu}$ measure is pairwise incomparable with each of the 
$I_{dr}$, $I_{in}$, $I_{win}$, $I_{wou}$, $I_{cc}$, $I_{wcc}$, $I_{ic}$,$I_{pr}$, $I_{ngr}$, and $I_{ust}$ measures.
\end{proposition}

\begin{proof}
From the differences in satisfaction of properties in Proposition \ref{prop:properties:structure}, Proposition \ref{prop:properties:extension}, and Proposition \ref{prop:properties:cu},  $I_{cu}$ is pairwise incompatible with $I_{win}$, $I_{wou}$, $I_{ic}$,$I_{pr}$, $I_{ngr}$, and $I_{ust}$ measures.
However, from the properties in Proposition \ref{prop:properties:structure}, we cannot discriminate $I_{cu}$ from $I_{in}$, $I_{wcc}$, and $I_{cc}$.
To discriminate $I_{cu}$ from $I_{in}$, consider the following graphs $G_1$ (left) and $G_2$ (right) where $I_{in}(G_1) = 2$ and $I_{in}(G_2) = 2$, whereas $I_{cu}(G_1) = 2$ and $I_{cu}(G_2) = 2/3$.
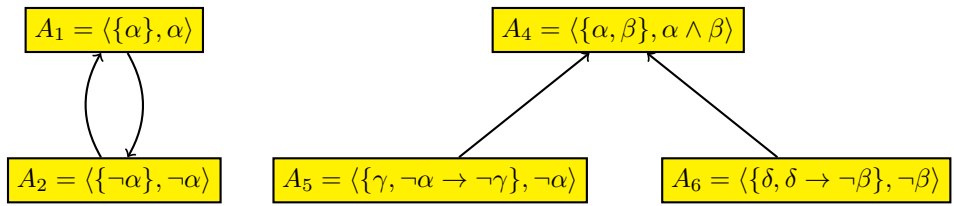
\begin{center}
\begin{tikzpicture}[->,thick]
\tikzstyle{every node}=[draw,rectangle,fill=yellow]
\node (A1)  at (0,2) {$A_1 = \langle\{ \alpha \}, \alpha \rangle$};
\node (A2)  at (0,0) {$A_2 = \langle\{ \neg\alpha  \}, \neg\alpha \rangle$};
%\node (A3)  at (0,0) {$A_3 = \langle\{ \delta, \neg\delta\rightarrow\neg\gamma \}, \neg\gamma \rangle$};
\path	(A2)[] edge[bend left] (A1);		
\path	(A1)[] edge[bend left] (A2);		
\end{tikzpicture}
\hspace{0.5cm}
\begin{tikzpicture}[->,thick]
\tikzstyle{every node}=[draw,rectangle,fill=yellow]
\node (A4)  at (2.5,2) {$A_4 = \langle\{ \alpha,\beta \}, \alpha\wedge\beta \rangle$};
\node (A5)  at (0,0) {$A_5 = \langle\{ \gamma, \neg\alpha\rightarrow\neg\gamma  \}, \neg\alpha \rangle$};
\node (A6)  at (5,0) {$A_6 = \langle\{ \delta, \delta\rightarrow\neg\beta \}, \neg\beta \rangle$};
\path	(A5)[] edge[] (A4);		
\path	(A6)[] edge[] (A4);		
\end{tikzpicture}
\end{center}
To discriminate $I_{cu}$ from $I_{cc}$, consider the following graphs $G_1$ (left) and $G_2$ (right) where $I_{cc}(G_1) = 1$ and $I_{cc}(G_2) = 1$, whereas $I_{cu}(G_1) = 2$ and $I_{cu}(G_2) = 2/3$.
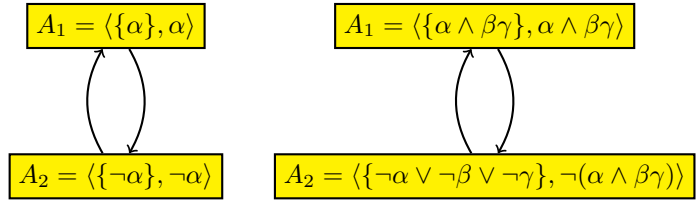
\begin{center}
\begin{tikzpicture}[->,thick]
\tikzstyle{every node}=[draw,rectangle,fill=yellow]
\node (A1)  at (0,2) {$A_1 = \langle\{ \alpha \}, \alpha \rangle$};
\node (A2)  at (0,0) {$A_2 = \langle\{ \neg\alpha  \}, \neg\alpha \rangle$};
%\node (A3)  at (0,0) {$A_3 = \langle\{ \delta, \neg\delta\rightarrow\neg\gamma \}, \neg\gamma \rangle$};
\path	(A2)[] edge[bend left] (A1);		
\path	(A1)[] edge[bend left] (A2);		
\end{tikzpicture}
\hspace{0.5cm}
\begin{tikzpicture}[->,thick]
\tikzstyle{every node}=[draw,rectangle,fill=yellow]
\node (A1)  at (0,2) {$A_1 = \langle\{ \alpha\wedge\beta\gamma \}, \alpha\wedge\beta\gamma \rangle$};
\node (A2)  at (0,0) {$A_2 = \langle\{ \neg\alpha\vee\neg\beta\vee\neg\gamma  \}, \neg(\alpha\wedge\beta\gamma) \rangle$};
%\node (A3)  at (0,0) {$A_3 = \langle\{ \delta, \neg\delta\rightarrow\neg\gamma \}, \neg\gamma \rangle$};
\path	(A2)[] edge[bend left] (A1);		
\path	(A1)[] edge[bend left] (A2);		
\end{tikzpicture}
\end{center}
To discriminate $I_{cu}$ from $I_{wcc}$, we can use a similar example to above.
\end{proof}

%%%%Another possibility for a degree of undercut comes from counting how many of the atoms occurring in the support of the argument are contradicted by the counterargument (what ${\sf Degree}_\mathsf{C}$ does may sometimes give quite a different result, some examples are given later). Consider the argument $\langle \{\alpha\}, \alpha \rangle$ and its canonical undercut  $\langle \{ \neg \alpha \wedge \neg \beta \wedge \neg \gamma\}, \Diamond \rangle$. The argument is fully undermined because it asserts nothing about $\beta$ and $\gamma$, but the only thing it asserts, namely $\alpha$, is contradicted by the counterargument. We could hold, informally, that the argument is $100\,\%$ contradicted.

In the proposals for degree of undercut \cite{BH05aaai,BH08book}, there are further options for degree of undercut, and these could be harnessed directly in the cumulative degree of undercut definition to provide potentially useful alternatives.

%%%%%%%%%%%%%%%%%%%%%%%%%%%%%%%%%%%%%%%%%%%%%%%%%%
\subsection{Application of logic-based measures of inconsistency}
%%%%%%%%%%%%%%%%%%%%%%%%%%%%%%%%

In this section, we harness logic-based inconsistency measures to evaluate the degree of inconsistency in an argument graphs instantiated with deductive arguments. We start by reviewing a couple of simple logic-based inconsistency measures. The first is the number of minimal inconsistent subsets of the knowledgebase, and the second is the sum of the inverse of the cardinality of each minimal inconsistent subset.

\begin{definition}
Let $K$ be a set of propositional classical logic formulae,
and let ${\sf MinIncon}(K)$ be the set of minimal inconsistent subsets of $K$.
The $I_M$ measure and the $I_{\#}$ measure are defined as follows.
\[
I_M(K) = |{\sf MinIncon}(K)|
\hspace{2cm}
I_{\#}(K) = \sum_{X \in {\sf MinIncon}(K)} \frac{1}{|X|}
\]
\end{definition}

\begin{example}
Let $K = \{ \alpha, \neg\alpha\vee\neg\beta, \beta, \neg\gamma, \neg\gamma\rightarrow\neg\alpha\}$.
So ${\sf MinIncon}(K)$ is as below, $I_M(K) = 2$ and $I_{\#}(K) = 2/3$. 
\[
{\sf MinIncon}(K) = \{     \{ \alpha, \neg\alpha\vee\neg\beta, \beta \}, 
													\{    \alpha,    \neg\gamma, \neg\gamma\rightarrow\neg\alpha\}          \}
\]
\end{example}

The cumulative attack inconsistency measure takes the sum of the degree of inconsistency of the premises of each attacker and attackee. 

\begin{definition}
The {\bf cumulative attack inconsistency measure} w.r.t logic-based inconsistency measure $I'$ is 
\[
I^{C}_{I'}(G) = \sum_{(A_i,A_j) \in {\sf Arcs}(G)} I'(\support(A_i) \cup \support(A_j))
\]
\end{definition}

\begin{example}
\label{ex:supportmeasure}
For the following graph $G$, $I^{C}_{I_M}(G) = 2$ and $I^{C}_{I_{\#}}(G) = 1/2$. 
\begin{center}
\begin{tikzpicture}[->,thick]
\tikzstyle{every node}=[draw,rectangle,fill=yellow]
\node (A1)  at (0,0) {$A_1 = \langle\{ \neg\alpha  \}, \neg\alpha \rangle$};
\node (A2)  at (5,0) {$A_2 = \langle\{  \alpha, \beta, \alpha\wedge\beta\rightarrow\gamma \}, \gamma \rangle$};
\node (A3) at (10,0) {$A_3 = \langle\{ \neg\beta  \}, \neg\beta\rangle$};
\path	(A1)[] edge[] (A2);		
\path	(A3)[] edge[] (A2);		
\end{tikzpicture}
\end{center}
\end{example}

The support inconsistency measure takes the degree of inconsistency of the premises of all the arguments in the graph taken together. 

\begin{definition}
The {\bf support inconsistency measure} w.r.t logical inconsistency measure $I'$ is 
\[
I^S_{I'}(G) = I'(\bigcup_{A \in \nodes(G)} \support(A))
\]
\end{definition}

\begin{example}
Continuing Example \ref{ex:supportmeasure}, $I^{S}_{I_M}(G) = 2$ and $I^{C}_{I_{\#}}(G) = 1/2$. 
\end{example}

\begin{example}
\label{ex:supportmeasure2}
For the following graph $G$, $I^{C}_{I_M}(G) = 2$, $I^{C}_{I_{\#}}(G) = 1/2$, $I^{S}_{I_M}(G) = 3$, and $I^{S}_{I_{\#}}(G) = 1$
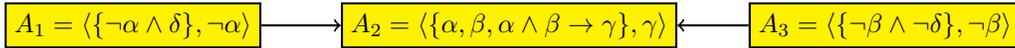
\begin{center}
\begin{tikzpicture}[->,thick]
\tikzstyle{every node}=[draw,rectangle,fill=yellow]
\node (A1)  at (0,0) {$A_1 = \langle\{ \neg\alpha\wedge\delta \}, \neg\alpha \rangle$};
\node (A2)  at (5,0) {$A_2 = \langle\{  \alpha, \beta, \alpha\wedge\beta\rightarrow\gamma \}, \gamma \rangle$};
\node (A3) at (10,0) {$A_3 = \langle\{ \neg\beta\wedge\neg\delta  \}, \neg\beta\rangle$};
\path	(A1)[] edge[] (A2);		
\path	(A3)[] edge[] (A2);		
\end{tikzpicture}
\end{center}
\end{example}

%\todo{Are results in this section dependent on choice of $I'$?}

\begin{proposition}
The $I^C_{I'}$ measure is a graph-based inconsistency measure according 
to Definition \ref{def:inconsistencymeasure}.
\end{proposition}

\begin{proof}
(Consistency) Assume ${\sf Arcs}(G) = \emptyset$. 
Therefore, $\sum_{(A_i,A_j) \in {\sf Arcs}(G)} I'(\support(A_i) \cup \support(A_j) = 0$.
Therefore, $I^C_{I'}(G) = 0$. 
(Freeness) 
Assume ${\sf Nodes}(G) = {\sf Nodes}(G')\setminus\{A\}$ 
and ${\sf Arcs}(G) = {\sf Arcs}(G')$. 
So, $\sum_{(A_i,A_j) \in {\sf Arcs}(G)} I'(\support(A_i) \cup \support(A_j))$
= $\sum_{(A_i,A_j) \in {\sf Arcs}(G')} I'(\support(A_i) \cup \support(A_j))$.
Therefore, $I^C_{I'}(G) = I^C_{I'}(G')$. 
\end{proof}

\begin{definition}
An argument graph $G$ is {\bf reflective} iff 
$\mbox{ if } \bigcup_{A \in \nodes(G)} \support(A) \vdash \bot, \mbox{ then } {\sf Arcs}(G) \neq \emptyset$. 
\end{definition}

\begin{assumption}
For the rest of the paper, we assume that all the argument graphs are reflective.
\end{assumption}

Despite having an intuitive rationale, $I^S_{I'}$ is not a graph-based inconsistency measures according 
to Definition \ref{def:inconsistencymeasure}.

\begin{proposition}
The $I^S_{I'}$ measure satisfies consistency but not freeness (as given in Definition \ref{def:inconsistencymeasure}). 
\end{proposition}

\begin{proof}
(Consistency) Assume ${\sf Arcs}(G) = \emptyset$. 
Therefore, there are no arguments $A_i$ and $A_j$ such that $A_i$ attacks $A_j$.
Therefore, $\bigcup_{A \in \nodes(G)} \support(A)\not\vdash\bot$. 
Therefore, $I^S_{I_M}(G) = 0$. 
(Freeness) 
Consider $A_1 = \langle\{\alpha\wedge\beta\},\alpha\leftrightarrow\beta\rangle$,
$A_2 = \langle\{\neg\alpha\wedge\gamma\},\alpha\leftrightarrow\beta\rangle$,
and $A_3 = \langle\{\neg\beta\vee\neg\gamma,\neg\beta\vee\neg\gamma\rightarrow\delta\},\delta\rangle$.
Let ${\sf Nodes}(G) = \{A_1,A_2\}$, ${\sf Arcs}(G) = \{(A_2,A_1)\}$, 
and ${\sf Nodes}(G') = \{A_1,A_2,A_3\}$.
So $I^S_{I_M}(G) = 1$ and $I^S_{I_M}(G') = 2$. 
%Assume ${\sf Nodes}(G) = {\sf Nodes}(G')\setminus\{A\}$ 
%and ${\sf Arcs}(G) = {\sf Arcs}(G')$. 
%So, $I'(\bigcup_{A \in \nodes(G)} \support(A))$ = $I'(\bigcup_{A \in \nodes(G')} \support(A))$. 
%Therefore, $I^C_{I_M}(G) = I^C_{I_M}(G')$. 
\end{proof}

\begin{assumption}
We assume for the rest of this paper that when an argument $A$ appears in ${\sf Nodes}(G)$ and in ${\sf Nodes}(G')$, then the logical argument associated with the node is the same (i.e. ${\sf Support}(A)$ is the same in both graphs, and ${\sf Claim}(A)$ is the same in both graphs). In addition, for argument $A \in {\sf Nodes}(G)$ and $A' \in {\sf Nodes}(G')$, if ${\sf Support}(A)$ = ${\sf Support}(A')$ and ${\sf Claim}(A)$ = ${\sf Claim}(A')$ then $A$ and $A'$ have the same name (i.e. $A = A'$). 
\end{assumption}

%%%%\todo{For following result - we assume e.g. ${\sf Arcs}(G) \subseteq {\sf Arcs}(G')$ means the content of the nodes is the same - this needs being made explicit. }

\begin{proposition}
The $I^C_{I'}$ measure satisfies Monotonicity,
Inversion,
Isomorphic invariance,
Disjoint additivity,
and 
Super-additivity.
\end{proposition}

\begin{proof}
(Monotonicity) 
Assume $G \sqsubseteq G'$. 
So ${\sf Arcs}(G) \subseteq {\sf Arcs}(G')$.
So, $\sum_{(A_i,A_j) \in {\sf Arcs}(G)} I'(\support(A_i) \cup \support(A_j)$
$=$ $\sum_{(A_i,A_j) \in {\sf Arcs}(G')} I'(\support(A_i) \cup \support(A_j)$.
So, $I_{I'}(G) \leq I_{I'}(G')$. 
(Inversion) Assume $G' = {\sf Invert}(G)$.
So,
\[
\sum_{(A_i,A_j) \in {\sf Arcs}(G)} I'(\support(A_i) \cup \support(A_j)
= \sum_{(A_i,A_j) \in {\sf Arcs}({\sf Invert}(G))} I'(\support(A_j) \cup \support(A_i)
\]
Therefore, $I_{I'}(G) = I_{I'}({\sf Invert}(G))$. 
(Isomorphic invariance) Similar to proof for inversion.
(Disjoint additivity) Assume $G_1$ and $G_2$ are disjoint.
Therefore, $\sum_{(A_i,A_j) \in {\sf Arcs}(G)} I'(\support(A_i) \cup \support(A_j)$ =
\[
\sum_{(A_i,A_j) \in {\sf Arcs}(G_1)} I'(\support(A_i) \cup \support(A_j)
+ \sum_{(A_i,A_j) \in {\sf Arcs}((G_2)} I'(\support(A_j) \cup \support(A_i)
\]
Therefore, $I_{I'}(G) = I_{I'}(G_1) + I_{I'}(G_2)$.
(Super-additivity) Similar to proof for disjoint additivity. 
\end{proof}

\begin{proposition}
The $I^S_{I'}$ measure satisfies Monotonicity,
Inversion,
Isomorphic invariance,
Disjoint additivity,
and 
Super-additivity.
\end{proposition}

\begin{proof}
(Monotonicity) 
Assume $G \sqsubseteq G'$. 
So ${\sf Arcs}(G) \subseteq {\sf Arcs}(G')$.
So 
\[
\bigcup_{A \in \nodes(G)} \support(A) \subseteq \bigcup_{A \in \nodes(G')} \support(A)
\]
So, $I^S_{I'}(G) \leq I^S_{I'}(G')$. 
(Inversion) Assume $G' = {\sf Invert}(G)$.
So, 
\[
\bigcup_{A \in \nodes(G)} \support(A) = \bigcup_{A \in \nodes(G')} \support(A)
\]
Therefore, $I^S_{I'}(G) \leq I^S_{I'}({\sf Invert}(G))$. 
(Isomorphic invariance) Similar to proof for inversion.
(Disjoint additivity) 
Consider $G_1$ (left) and $G_2$ (right) where $I_{I_M}(G_1+G_2) = 2$
$I_{I_M}(G_1) = 2$ and $I_{I_M}(G_2) = 2$.
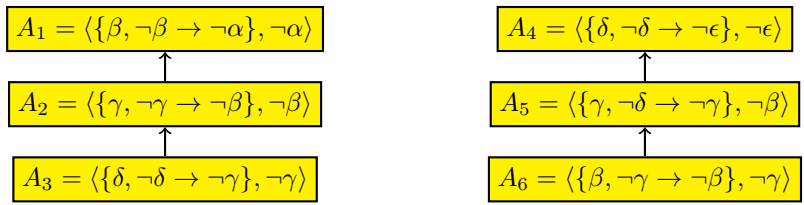
\begin{center}
\begin{tikzpicture}[->,thick]
\tikzstyle{every node}=[draw,rectangle,fill=yellow]
\node (A1)  at (0,2) {$A_1 = \langle\{ \beta, \neg\beta\rightarrow\neg\alpha \}, \neg\alpha \rangle$};
\node (A2)  at (0,1) {$A_2 = \langle\{ \gamma, \neg\gamma\rightarrow\neg\beta \}, \neg\beta \rangle$};
\node (A3)  at (0,0) {$A_3 = \langle\{ \delta, \neg\delta\rightarrow\neg\gamma \}, \neg\gamma \rangle$};
\path	(A2)[] edge[] (A1);		
\path	(A3)[] edge[] (A2);		
\end{tikzpicture}
\hspace{2cm}
\begin{tikzpicture}[->,thick]
\tikzstyle{every node}=[draw,rectangle,fill=yellow]
\node (A4)  at (0,2) {$A_4 = \langle\{ \delta,  \neg\delta\rightarrow\neg\epsilon \}, \neg\epsilon \rangle$};
\node (A5)  at (0,1) {$A_5 = \langle\{ \gamma, \neg\delta\rightarrow\neg\gamma  \}, \neg\beta \rangle$};
\node (A6)  at (0,0) {$A_6 = \langle\{ \beta, \neg\gamma\rightarrow\neg\beta \}, \neg\gamma \rangle$};
\path	(A5)[] edge[] (A4);		
\path	(A6)[] edge[] (A5);		
\end{tikzpicture}
\end{center}
(Super-additivity) Similar to proof for disjoint additivity. 
\end{proof}

%\begin{proposition}
%For all instantiated argument graphs $G$, $I^C_{I_M}(G) = I_{in}(G)$.
%\end{proposition}

%\begin{proof}
%For all the attack relations in Definition \ref{def:classical:attack}, $A$ attacks $B$ iff 

For pairwise incompatibility, we consider $I' = I_{M}$ below. We can obtain similar results for other instantiations of $I^C_{I'}$ and $I^S_{I'}$.

\begin{proposition}
The $I^C_{I_M}$ and $I^S_{I_M}$ measures are pairwise incompatible with each of the 
$I_{dr}$, $I_{in}$, $I_{win}$, $I_{wou}$, $I_{cc}$, $I_{wcc}$, $I_{ic}$,$I_{pr}$, $I_{ngr}$, $I_{ust}$, and $I_{cu}$ measures.
\end{proposition}

\begin{proof}
From the differences in satisfaction of properties in Proposition \ref{prop:properties:structure}, Proposition \ref{prop:properties:extension}, and Proposition \ref{prop:properties:cu},  $I^C_{I_M}$ is pairwise incompatible with $I_{win}$, $I_{wou}$, $I_{ic}$,$I_{pr}$, $I_{ngr}$, and $I_{ust}$ measures.
However, from the properties in Proposition \ref{prop:properties:structure}, we cannot discriminate $I^C_{I_M}$  from $I_{in}$, $I_{cc}$, $I_{wcc}$, and $I_{cu}$.
To discriminate $I^C_{I_M}$ from $I_{in}$, consider the following graphs $G_1$ (left) and $G_2$ (right) where $I_{in}(G_1) = 2$ and $I_{in}(G_2) = 1$, whereas $I^C_{I_M}(G_1) = 2$ and $I^C_{I_M}(G_2) = 2$.
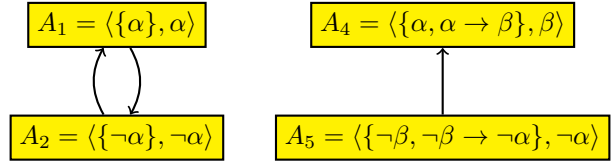
\begin{center}
\begin{tikzpicture}[->,thick]
\tikzstyle{every node}=[draw,rectangle,fill=yellow]
\node (A1)  at (0,1.5) {$A_1 = \langle\{ \alpha \}, \alpha \rangle$};
\node (A2)  at (0,0) {$A_2 = \langle\{ \neg\alpha  \}, \neg\alpha \rangle$};
%\node (A3)  at (0,0) {$A_3 = \langle\{ \delta, \neg\delta\rightarrow\neg\gamma \}, \neg\gamma \rangle$};
\path	(A2)[] edge[bend left] (A1);		
\path	(A1)[] edge[bend left] (A2);		
\end{tikzpicture}
\hspace{0.5cm}
\begin{tikzpicture}[->,thick]
\tikzstyle{every node}=[draw,rectangle,fill=yellow]
\node (A4)  at (0,1.5) {$A_4 = \langle\{ \alpha,\alpha\rightarrow\beta \}, \beta \rangle$};
\node (A5)  at (0,0) {$A_5 = \langle\{ \neg\beta, \neg\beta\rightarrow\neg\alpha  \}, \neg\alpha \rangle$};
%\node (A6)  at (5,0) {$A_6 = \langle\{ \delta, \delta\rightarrow\neg\beta \}, \neg\beta \rangle$};
\path	(A5)[] edge[] (A4);		
%\path	(A6)[] edge[] (A4);		
\end{tikzpicture}
\end{center}
To discriminate $I^C_{I_M}$ from $I_{cc}$, consider the above graphs $G_1$ (above left) and $G_2$ (above right) where $I_{cc}(G_1) = 1$ and $I_{cc}(G_2) = 0$, whereas $I^C_{I_M}(G_1) = 1$ and $I^C_{I_M}(G_2) = 2$.
To discriminate $I^C_{I_M}$ from $I_{wcc}$ and from $I_{cu}$, we can use a similar example to above.
To show $I^S_{I_M}$ measure is pairwise incompatible with each of the 
$I_{dr}$, $I_{in}$, $I_{win}$, $I_{wou}$, $I_{cc}$, $I_{wcc}$, $I_{ic}$,$I_{pr}$, $I_{ngr}$, $I_{ust}$, and $I_{cu}$ measures, we can use the failure of freeness to create examples where order-compatibility fails for each pairwise comparison. 
\end{proof}

%\item Degree of attack
%\[
%D^{attack}_I(A,B) = I(\support(A),\support(B))
%\]

%\todo{Discuss following - degree of extension inconsistency}
%\[
%D^{extension}_I(G) = I(\bigcup_{A \in \extensions(G)} \support(A))
%\]

In this subsection, we have harnessed two existing logic-based inconsistency measures, $I_M$ and $I_{\#}$, for measuring inconsistency in argument graphs instantiated with deductive arguments. There is a wide range of further measures of inconsistency that we could deploy in this role (for reviews see \cite{GrantHunter2011,Thimm2016}). In addition, we have only considered two ways of applying logic-based measures, namely $I^C_{I'}$ and $I^S_{I'}$. Further, ways of applying logic-based measures include analysing the support in extensions of an argument graph to identify inconsistency. For instance, it is not necessarily the case the union of the support of the arguments in an extension is consistent (for more discussion of this point, see \cite{GorogiannisHunter2011}). Another option is to check the degree of inconsistency of the premises of the defenders of an argument since it is not necessarily the case that these would be consistent. We leave investigation of these options to further work.

%%%%%%%%%%%%%%%%%%%%%%%%%%%%%%%%%%%%%%%%%%%%%%%%%%
\subsection{Related work}
%%%%%%%%%%%%%%%%%%%%%%%%%%%%%%%%%%%%%%%

In the converse of what we have considered in this section, deductive argumentation has been used for measuring inconsistency. For this, an argument tree (as defined by \cite{BH01,BH08book}) is used. Each node in the tree is an argument. Each child is a canonical undercut of its parents. For each node in the tree, each canonical undercut of the node is a child of the node (except when the premises of the child have all occurred in the support of the argument in ancestor arguments). This exception prohibits infinite branches where each argument has the same premises that have already occurred on the branch. 

\begin{example}
\label{ex:related}
Consider the knowledgebase $K = \{ \alpha, \neg\alpha, \neg\alpha\vee\beta,\beta,\neg\beta   \}$. 
The following is an argument tree with $A_1$ being the root.
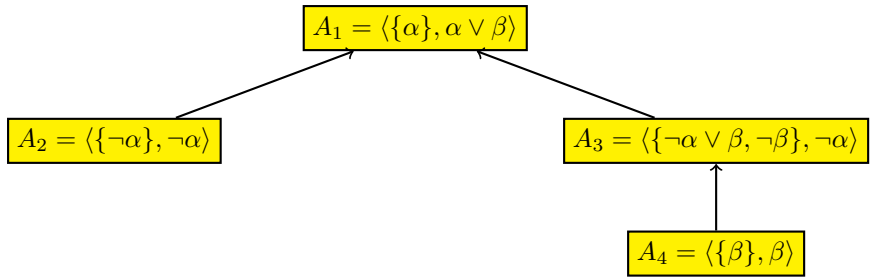
\begin{center}
\begin{tikzpicture}[->,thick]
\tikzstyle{every node}=[draw,rectangle,fill=yellow]
\node (A1)  at (4,3) {$A_1 = \langle\{ \alpha \}, \alpha\vee\beta \rangle$};
\node (A2)  at (0,1.5) {$A_2 = \langle\{ \neg\alpha  \}, \neg\alpha \rangle$};
\node (A3)  at (8,1.5) {$A_3 = \langle\{ \neg\alpha\vee\beta, \neg\beta \}, \neg\alpha \rangle$};
\node (A4)  at (8,0) {$A_4 = \langle\{ \beta  \}, \beta \rangle$};
\path	(A2)[] edge[] (A1);		
\path	(A3)[] edge[] (A1);		
\path	(A4)[] edge[] (A3);		
\end{tikzpicture}
\end{center}
\end{example}

In \cite{Raddaoui2015}, an argument tree is constructed where the argument at the root has a single premise. Then, three proposals are made for evaluating the inconsistency of this formula.

\begin{definition}
Let $T$ be the argument tree with the root argument $A$ having ${\sf Support}(A) = \alpha)$.
An $I^x_{ARG}$ measure is defined as follows
\[
I^x_{ARG}(\alpha,T) = |{\sf Undercuts}(\alpha,T)| \times f^x(\alpha,T)
\]
where ${\sf Undercuts}(\alpha,T)$ is the set of undercuts of the root argument in $T$, ${\sf Height}(T)$ is the height of the tree, and ${\sf Depth}(n,T)$ is the depth of node $n$ in the tree. 
\[
f^1(\alpha,T) = \frac{1}{{\sf Height}(T)}
\hspace{4mm}
f^2(\alpha,T) = \frac{1}{\sum_{n \in {\sf Nodes}(T)}  {\sf Depth}(n,T)  }
\hspace{4mm}
f^3(\alpha,T) = \sum_{n \in {\sf Nodes}(T)}  \frac{1}{{\sf Depth}(n,T)  }
\]
\end{definition}

So $I^1_{ARG}(\alpha,T)$ takes the height of the tree into account, $I^2_{ARG}(\alpha,T)$ takes inverse of the sum of the depth of each node into account, and $I^3_{ARG}(\alpha,T)$ takes the sum of the inverse of the depth of each node into account. 

\begin{example}
Continuing Example \ref{ex:related}, the measures for three of the formulae are tabulated.
\begin{center}
\begin{tabular}{|c|c|c|c|}
\hline
& $\alpha$ & $\neg\alpha\vee\beta$ & $\neg\alpha$\\
\hline
\hline
$I^1_{ARG}$   & 1 & 1/2 & 1/3 \\
$I^2_{ARG}$   & 1/2 & 1/5 & 1/6 \\
$I^3_{ARG}$   & 5 & 2 & 11/6\\
\hline
\end{tabular}
\end{center}
\end{example}

Whilst the proposal by Raddaoui \cite{Raddaoui2015} is for measuring inconsistency in a formula, it possible that the ideas could be adapted for measuring inconsistency of argument graphs. 

%%%%%%%%%%%%%%%%%%%%%%%%%%%%%%%%%%%%%%%%%%%%%%%%%%
%%%%%%%%%%%%%%%%%%%%%%%%%%%%%%%%%%%%%%%%%%%%%%%%%%
%%%%%%%%%%%%%%%%%%%%%%%%%%%%%%%%%%%%%%%%%%%%%%%%%%
\section{Resolution through commitment}
%%%%%%%%%%%%%%%%%%%%%%%%%%%%%%%%%%%%%%%
\label{section:resolution}

An agent can commit to some arguments (i.e. declare whether they think an argument is acceptable or not), and s/he can be queried about those commitments. We assume that commitment by an agent is represented by the belief an agent has in the arguments. 

In this section, we consider how we can model the agent. 
For this, we assume a labelling function as defined in Section \ref{section:labelsemantics}. Initially, if we know nothing about the agent, we start with a uniform labelling that assigns $\argundec$ to each argument.  Then suppose the  agent declares a commitment in an argument --- either by saying that the label for the argument is $\argin$ or that it is $\argout$ - we can consider what the ramifications are of that commitment on the other beliefs, and moreover, we can use it for resolving inconsistency (i.e. reducing the measure of inconsistency of the argument graph). 

We proceed by introducing some subsidiary definitions for labellings, and for generating a subgraph of a graph based on a labelling. The first subsidiary definition specifies a labelling for which there is no undecided labels.

%We update this labelling function according to the following process.
%\begin{enumerate}
%\item Start with a labelling function $L_1$ which labels all arguments in $G_1$ with $\argundec$.
%\item We choose an argument $A$ in $G$ on which to seek more information (implicitly we assume that %there is some cost involved in getting more information) and so we want to choose an argument that when we know the answer, will help decrease the inconsistency. 
%\begin{itemize} 
%\item With the answer, we update $L_1$ to $L_2$ by changing the assignment for $A$ depending on the answer. So if we believe $A$, then we change the assignment to $\argin$, and if we do not believe $A$, then we change the assignment to $\argout$. \todo{Do we update any other arguments based on the new information using dialectical semantics?}
%\item Generate $G_2$ from $G_1$ by deleting all arguments (and any attacks that involve that argument) that are labelled $\argout$.
%\end{itemize}
%\item We repeat the previous step to update $L_i$ to $L_{i+1}$ and $G_i$ to $G_{i+1}$ until the inconsistency is sufficiently low in $G_{i+1}$ or until we have a complete labelling. 
%\end{enumerate}

%\todo[inline]{The above process needs to be formalized - and there seem to be some options here - particularly for determination}

\begin{definition}
A labelling $L$ is {\bf committed} for graph $G$ iff for all $A \in {\sf Nodes}(G)$, $A \in \argin(G)$ or  $A \in \argout(G)$. 
\end{definition}

\begin{example}
\label{ex:committed}
Consider the following graph. For this, the following labelling is committed: $L(A1) = \argout$, $L(A2) = \argin$, and $L(A3) = \argout$.
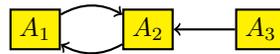
\begin{center}
\begin{tikzpicture}[->,thick]
\tikzstyle{every node}=[draw,rectangle,fill=yellow]
\node (A1)  at (0,0) {$A_1$};
\node (A2)  at (1.5,0) {$A_2$};
\node (A3) at (3,0) {$A_3$};
\path	(A1)[] edge[bend left] (A2);		
\path	(A2)[] edge[bend left] (A1);
\path	(A3)[] edge[] (A2);
\end{tikzpicture}
\end{center}
\end{example}

The next subsidiary definition constrains a labelling to take into account the attack relationship. As we illustrate in the subsequent example, a strict labelling is not necessarily an admissible labelling, though every admissible labelling is a strict labelling.

\begin{definition}
A labelling $L$ is {\bf strict} for graph $G$ iff for all $(A,B) \in {\sf Arcs}(G)$, if $A \in \argin(G)$, then  $B \in \argout(G)$. 
\end{definition}

%\begin{example}
%Consider the graph in Example \ref{ex:committed}. For this, the labelling
%$L(A_1) = \argundec$, $L(A_2) = \argin$, and $L(A_1) = \argout$, is a strict labelling but not an admissible. 
%Whereas the labelling $L'(A_1) = \argundec$, $L'(A_2) = \argundec$, and $L'(A_1) = \argout$, is a strict labelling but not an admissible. 
% \end{example}

\begin{example}
Consider the graph in Example \ref{ex:committed}. For this, the labellings that are committed and strict are tabulated below. 
%Note, that $L_1$ to $L_4$ are admissible, but $L_5$ is not admissible. 
\begin{center}
\begin{tabular}{|c|c|c|c|}
\hline
& $A_1$ & $A_2$ & $A_3$ \\
\hline
\hline
$L_1$ & $\argout$ & $\argout$& $\argout$\\
$L_2$ & $\argin$ & $\argout$& $\argout$\\
$L_3$ & $\argout$ & $\argout$& $\argin$\\
$L_4$ & $\argin$ & $\argout$& $\argin$\\
$L_5$ & $\argout$ & $\argin$& $\argout$\\
\hline
\end{tabular}
\end{center}
\end{example}

The following definition forms a subgraph from a graph and a labelling by deleting every node that is labelled $\argout$ and deleting every arc that has either the source or the target labelled $\argout$. The reason we want this new graph is that if an agent commits to an argument being $\argout$, then that argument is no longer acceptable, and we can ignore it from further consideration.

\begin{definition}
Given an argument graph $G$ and labelling $L$, the {\bf new graph} function is ${\sf NewGraph}(G,L) = G'$ where 
\begin{itemize}
\item ${\sf Nodes}(G') = \{A \in {\sf Nodes}(G) \mid L(A) \neq \argout   \}$
\item ${\sf Arcs}(G') = \{ (A,B) \in {\sf Arcs}(G) \mid L(A) \neq \argout \mbox{ or }  L(B) \neq \argout  \}$
\end{itemize}
\end{definition}

\begin{example}
Consider the following graph $G$ with the labelling 
$L(A_1) = \argundec$,
$L(A_2) = \argundec$,
$L(A_3) = \argout$,
$L(A_4) = \argin$,
$L(A_5) = \argout$,
$L(A_6) = \argin$,
and 
$L(A_7) = \argundec$.
\begin{center}
\begin{tikzpicture}[->,thick]
\tikzstyle{every node}=[draw,rectangle,fill=yellow]
\node (A1)  at (0,0) {$A_1$};
\node (A2)  at (1.5,1) {$A_2$};
\node (A3) at (1.5,0) {$A_3$};
\node (A4)  at (3,0) {$A_4$};
\node (A5)  at (4.5,0) {$A_5$};
\node (A6)  at (4.5,1) {$A_6$};
\node (A7) at (6,0) {$A_7$};
\path	(A1)[] edge[bend left] (A2);		
\path	(A1)[] edge[] (A3);
\path	(A2)[] edge[bend left] (A4);
\path	(A3)[] edge[] (A4);
\path	(A3)[] edge[bend right] (A5);
\path	(A4)[] edge[] (A5);
\path	(A4)[] edge[bend left] (A6);
\path	(A5)[] edge[] (A7);
\path	(A6)[] edge[bend left] (A7);
\end{tikzpicture}
\end{center}
So the new graph $G'$ for this graph and labelling is below.
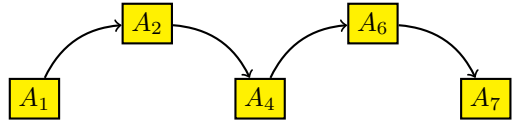
\begin{center}
\begin{tikzpicture}[->,thick]
\tikzstyle{every node}=[draw,rectangle,fill=yellow]
\node (A1)  at (0,0) {$A_1$};
\node (A2)  at (1.5,1) {$A_2$};
%\node (A3) at (1.5,0) {$A_3$};
\node (A4)  at (3,0) {$A_4$};
%\node (A5)  at (4.5,0) {$A_5$};
\node (A6)  at (4.5,1) {$A_6$};
\node (A7) at (6,0) {$A_7$};
\path	(A1)[] edge[bend left] (A2);		
%\path	(A1)[] edge[] (A3);
\path	(A2)[] edge[bend left] (A4);
%\path	(A3)[] edge[] (A4);
%\path	(A3)[] edge[bend right] (A5);
%\path	(A4)[] edge[] (A5);
\path	(A4)[] edge[bend left] (A6);
%\path	(A5)[] edge[] (A7);
\path	(A6)[] edge[bend left] (A7);
\end{tikzpicture}
\end{center}
We can apply a measure of inconsistency to the graph $G$ and $G'$, and determine the reduction in inconsistency. For instance, $I_{in}(G) = 9$ and $I_{in}(G') = 4$. Similarly, $I_{win}(G) = 9/2$ and $I_{win}(G') = 4$.
\end{example}

When we query an agent about an argument $A$, we get a reply of either $\argin$ or $\argout$. Given this information, we need to update the labelling to $L'$ as follows.

\begin{definition}
Let $L$ be a labelling for graph $G$, and let $A \in {\sf Nodes}(G)$ be a query.
\begin{itemize}
\item If the answer for $A$ is $\argin$, then 
\begin{itemize}
\item $L'(A) = \argin$
\item for each $(A,B) \in {\sf Nodes}(G)$, $L'(B) = \argout$,
\item for each $(B,A) \in {\sf Nodes}(G)$, $L'(B) = \argout$,
\item for all other arguments $C$, $L'(C) = L(C)$.
\end{itemize}
\item If the answer for $A$ is $\argout$, then 
\begin{itemize}
\item $L'(A) = \argout$
%\item for each $(A,B) \in {\sf Nodes}(G)$, $L'(B) = \argout$,
%\item for each $(B,A) \in {\sf Nodes}(G)$, $L'(B) = \argout$,
\item for all other arguments $C$, $L'(C) = L(C)$.
\end{itemize}
\end{itemize}
\end{definition}

\begin{example}
Consider the following graph. Let $L_1$ be the original labelling, let $L_2$ be the new labelling obtained after the first query, and let $L_3$ be the new labelling obtained after the second query.
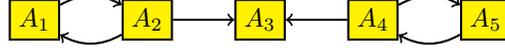
\begin{center}
\begin{tikzpicture}[->,thick]
\tikzstyle{every node}=[draw,rectangle,fill=yellow]
\node (A1)  at (0,0) {$A_1$};
\node (A2)  at (1.5,0) {$A_2$};
\node (A3) at (3,0) {$A_3$};
\node (A4) at (4.5,0) {$A_4$};
\node (A5) at (6,0) {$A_5$};
\path	(A1)[] edge[bend left] (A2);		
\path	(A2)[] edge[bend left] (A1);
\path	(A2)[] edge[] (A3);
\path	(A4)[] edge[] (A3);
\path	(A4)[] edge[bend left] (A5);		
\path	(A5)[] edge[bend left] (A4);
\end{tikzpicture}
\end{center}
Suppose the first query concerns $A_1$, with the reply $\argout$, and the second query concerns $A_4$, with the reply $\argin$. The labellings are tabulated below.
\begin{center}
\begin{tabular}{|l|ccccc|}
\hline
& $A_1$ & $A_2$ & $A_3$ & $A_4$ & $A_5$ \\
\hline
\hline
$L_1$ & $\argundec$ & $\argundec$ & $\argundec$ & $\argundec$ & $\argundec$ \\
$L_2$ & $\argout$ & $\argundec$ & $\argundec$ & $\argundec$ & $\argundec$ \\
$L_3$ & $\argout$ & $\argundec$ & $\argout$ & $\argin$ & $\argout$ \\
\hline
\end{tabular}
\end{center}
\end{example}

Now we can show how measures of inconsistency can help in deciding which arguments to query. We illustrate this in the following example.

\begin{example}
\label{ex:committed:3}
Consider the following graph. 
Suppose $L_0(A_i) = \argundec$ for all $A_i \in \{A_1,\ldots,A_5\}$.
We could query any of these arguments. 

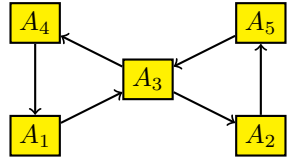
\begin{center}
\begin{tikzpicture}[->,thick]
\tikzstyle{every node}=[draw,rectangle,fill=yellow]
\node (A1)  at (0,0) {$A_1$};
\node (A2)  at (3,0) {$A_2$};
\node (A3) at (1.5,0.75) {$A_3$};
\node (A4) at (0,1.5) {$A_4$};
\node (A5) at (3,1.5) {$A_5$};
\path	(A1)[] edge[] (A3);		
\path	(A3)[] edge[] (A2);		
\path	(A3)[] edge[] (A4);		
\path	(A5)[] edge[] (A3);
\path	(A4)[] edge[] (A1);
\path	(A2)[] edge[] (A5);
\end{tikzpicture}
\end{center}
In the following, each bullet point concerns a specific query and specific answer to that query. In each case, given a graph $G$, and the revised labelling $L$, we obtain the new graph $G'$ 
where ${\sf NewGraph}(G,L) = G'$.
\begin{itemize}
\item If we query $A_3$, and we get the answer $A_3$ is $\argin$, then the resulting labelling is 
\[
\begin{array}{ccccc}
L(A_1) = \argout & L(A_2) = \argout & L(A_3) = \argin & L(A_4) = \argout & L(A_5) = \argout 
\end{array}
\]
\begin{itemize}
\item Hence, ${\sf NewGraph}(G,L) = ( \{  A_3 \}, \{ \} )$.
\item So $I_{in}(G') = 0$ and $I_{cc}(G') = 0$.
\end{itemize}
\item If we query $A_3$, and we get the answer $A_3$ is $\argout$, then the resulting labelling is 
\[
\begin{array}{ccccc}
L(A_1) = \argundec & L(A_2) = \argundec & L(A_3) = \argout & L(A_4) = \argundec & L(A_5) = \argundec 
\end{array}
\]
\begin{itemize}
\item Hence, ${\sf NewGraph}(G,L) = ( \{  A_1, A_2, A_4, A_5 \}, \{ \} )$.
\item So $I_{in}(G') = 2$ and $I_{cc}(G') = 0$.
\end{itemize}
\item If we query $A_1$, and we get the answer $A_1$ is $\argin$, then the resulting labelling is 
\[
\begin{array}{ccccc}
L(A_1) = \argin & L(A_2) = \argundec & L(A_3) = \argout & L(A_4) = \argout & L(A_5) = \argundec 
\end{array}
\]
\begin{itemize}
\item Hence, ${\sf NewGraph}(G,L) = ( \{  A_1, A_2, A_5 \}, \{ (A_2,A_5) \} )$.
\item So $I_{in}(G') = 1$ and $I_{cc}(G') = 0$.
\end{itemize}
\item If we query $A_1$, and we get the answer $A_1$ is $\argout$, then the resulting labelling is 
\[
\begin{array}{ccccc}
L(A_1) = \argout & L(A_2) = \argundec & L(A_3) = \argundec & L(A_4) = \argundec & L(A_5) = \argundec 
\end{array}
\]
\begin{itemize}
\item Hence, ${\sf NewGraph}(G,L) = ( \{  A_2, A_3, A_4, A_5 \}, \{ (A_3,A_4), (A_2,A_5), (A_3,A_2),(A_5,A_3) \} )$.
\item So $I_{in}(G') = 4$ and $I_{cc}(G') = 1$.
\end{itemize}
\item If we query $A_2$, and we get the answer $A_2$ is $\argin$, then the resulting labelling is 
\[
\begin{array}{ccccc}
L(A_1) = \argundec & L(A_2) = \argin & L(A_3) = \argout & L(A_4) = \argundec & L(A_5) = \argout 
\end{array}
\]
\begin{itemize}
\item Hence, ${\sf NewGraph}(G,L) = ( \{  A_1, A_2, A_4 \}, \{ (A_4,A_1)\} )$.
\item So $I_{in}(G') = 1$ and $I_{cc}(G') = 0$.
\end{itemize}
\item If we query $A_2$, and we get the answer $A_2$ is $\argout$, then the resulting labelling is 
\[
\begin{array}{ccccc}
L(A_1) = \argundec & L(A_2) = \argout & L(A_3) = \argundec & L(A_4) = \argundec & L(A_5) = \argundec 
\end{array}
\]
\begin{itemize}
\item Hence, ${\sf NewGraph}(G,L) = ( \{  A_1, A_3,A_4 A_5 \}, \{ (A_1,A_3),(A_3,A_4),(A_4,A_1),(A_5,A_3)\} )$.
\item So $I_{in}(G') = 4$ and $I_{cc}(G') = 1$.
\end{itemize}
\end{itemize}
Therefore, if we query $A_3$, we will have the maximum reduction in inconsistency if we use the $I_{in}$ or $I_{cc}$ measures where the reduction is the average of the reduction for the $\argin$ and $\argout$ answers. 
\end{example}

%\todo{Expand on the role of the inconsistency measure - apply to reduced graph - show equivalence with labelling}

%\todo{What properties does this process have + how does it relate to dialectical semantics?}

In this section, we have seenn how we can use the inconsistency measures for graphs as a way of guiding the selecting of arguments to query an agent. The answer from the agent is a commitment to accepting or not accepting the argument, and this can be used to resolve inconsistencies in the graph.

%%%%%%%%%%%%%%%%%%%%%%%%%%%%%%%%%%%%%%%%%%%%%%%%%%
%%%%%%%%%%%%%%%%%%%%%%%%%%%%%%%%%%%%%%%%%%%%%%%%%%
%%%%%%%%%%%%%%%%%%%%%%%%%%%%%%%%%%%%%%%%%%%%%%%%%%
\section{Discussion}
%%%%%%%%%%%%%%%
\label{section:discussion}

This paper makes the following contributions:
(1) A proposal for a general framework of postulates for characterizing measures of inconsistency for argument graphs;
(2) Proposals for graph structure measures and graph extension measures as instances of measures of inconsistency for argument graphs;
(3) A review of the degree of undercut approach to measuring inconsistency for argument graphs instantiated with deductive arguments;
(4) An investigation of the use of existing (logic-based) measures of inconsistency to measuring inconsistency for argument graphs instantiated with deductive arguments;
and (5) An outline of how measures of inconsistency for argument graphs can be used as part of process for inconsistency resolution in argumentation.

%%%\todo{Check for other relevant papers}

In future work, it would be good to give further properties for an inconsistency measure and relationships between them, further definitions for an inconsistency measure, further results for specific classes of graph, and methods for resolution based on commitments. 

The structure-based measures are in a sense measuring aspects of graph complexity. There are numerous options for features of argument graphs that could be considered (see for instance, features used for selecting argument solvers \cite{Vallati2014}).  Further options for measuring graph complexity include measuring the sparseness of the graph (e.g.  average indegree, average outdegree, indegree distribution, or outdegree distribution) which may give more recognition to unattacked arguments, radius of the graph (which is the maximum eccentricity of any  node in the graph where the eccentricity of the a node is the length of the shortest path to the node furthest away from that node), and dimensions of the graph (i.e. the minimum number of dimensions of Euclidean space required to represent all arcs with unit length). 
Another possible field for options for graph-based measures of inconsistency is graph entropy (see for example \cite{Dehmer11}). 

A comparison with proposals for argument strength would be interesting. These consider a weight to individual arguments which can be affected by arguments that impinge upon it. A number of proposals have been made (e.g. \cite{BH01,Cayrol:2005,Matt:2008,Amgoud:2013,Amgoud:2015a,Amgoud:2016,Amgoud:2016b,Grossi:2015a,Thimm:2014c,Bonzon:2016}), and comparisons undertake with a range of postulates (for a review see \cite{Bonzon:2016a}). Possibly, analogous postulates can be proposed for inconsistency measures in graphs. 

It would also be interesting to consider developing inconsistency measures for alternatives to deductive argumentation for structured argumentation such assumption-based argumentation (for a tutorial, see \cite{Toni2014tutorial}), defeasible logic programming (for a tutorial, see \cite{GarciaSimari2014tutorial}), and ASPIC+ (for a tutorial, see \cite{ModgilPrakken2014tutorial}).  This would involve developing logic-based inconsistency measures for these non-standard logical formalisms.

Finally, it would be valuable to apply these techniques as part of process for inconsistency resolution in argumentation for an application such as intelligence analysis to investigate the usability of the measures, and whether indeed there are tangible benefits to using these inconsistency measures.

%\todo{More on logic-based inconsistency measures - perhaps include postulates - and more references to literature}

%\todo{Need to add paragraph on WAF - with inconsistency budget - and developments e.g. Bistarelli with cost-based argumentation}

%\todo{Add references to graph complexity textbook or reviews}

%%%%%%%%%%%%%%%%%%%%%%%%%%%%%%%%%%%%%%%%%%%%%%%%%%
%%%%%%%%%%%%%%%%%%%%%%%%%%%%%%%%%%%%%%%%%%%%%%%%%%
%%%%%%%%%%%%%%%%%%%%%%%%%%%%%%%%%%%%%%%%%%%%%%%%%%
\subsection*{Acknowledgments}
%%%%%%%%%%%%%%%%%%%%%%%%%%%%%

This research was partly funded by EPSRC grant EP/N008294/1 for the Framework for Computational Persuasion project

%%%%%%%%%%%%%%%%%%%%%%%%%%%%%%%%%%%%%%%%%%%%%%%%%%
%%%%%%%%%%%%%%%%%%%%%%%%%%%%%%%%%%%%%%%%%%%%%%%%%%
%%%%%%%%%%%%%%%%%%%%%%%%%%%%%%%%%%%%%%%%%%%%%%%%%%

%\bibliographystyle{alpha}
%\bibliography{graphinc}

\newcommand{\etalchar}[1]{$^{#1}$}

\end{document}